\documentclass{article}

\usepackage{iclr2027_conference,times}
\iclrfinalcopy

\usepackage[utf8]{inputenc}
\usepackage[T1]{fontenc}
\usepackage{microtype}
\usepackage[dvipsnames,table]{xcolor}
\usepackage{bbm}
\usepackage{dsfont}
\usepackage{enumitem}
\usepackage{url}

\definecolor{VibrantPink}{HTML}{C2185B}
\definecolor{RoyalBlue}{HTML}{002366}
\definecolor{LightBlue}{HTML}{1E90FF}
\definecolor{AccentTeal}{HTML}{0F766E}
\definecolor{AccentAmber}{HTML}{D97706}
\definecolor{BgTeal}{HTML}{F3F7F7}
\definecolor{StrongBlue}{HTML}{1F4E79}
\definecolor{StrongBg}{HTML}{EEF4FA}
\definecolor{stdgray}{gray}{0.45}
\definecolor{best}{RGB}{0,0,0}

\usepackage{amsmath,amssymb,amsthm,mathtools}
\DeclareMathOperator*{\argmax}{arg\,max}

\theoremstyle{plain}
\newtheorem{theorem}{Theorem}
\newtheorem{proposition}[theorem]{Proposition}
\newtheorem{lemma}[theorem]{Lemma}
\newtheorem{corollary}[theorem]{Corollary}

\theoremstyle{definition}

\newtheorem{assumption}{Assumption}

\theoremstyle{remark}
\newtheorem{remark}[theorem]{Remark}

\usepackage{graphicx}
\usepackage{subcaption}
\usepackage{wrapfig}
\usepackage{needspace}
\usepackage{multirow}
\usepackage{booktabs}
\usepackage{algorithm}
\usepackage{algorithmic}

\usepackage[most]{tcolorbox}

\newtcolorbox{callout}[1][]{
    breakable, enhanced, colback=BgTeal, colframe=AccentTeal,
    boxrule=0.5pt, arc=2mm, left=1.5mm, right=1.5mm, top=1mm, bottom=1mm, #1
}
\newtcolorbox{calloutimportant}[1][]{
    breakable, enhanced, colback=StrongBg, colframe=StrongBlue,
    boxrule=0.6pt, arc=2mm, left=1.8mm, right=1.5mm, top=1mm, bottom=1mm,
    borderline west={2.2pt}{0pt}{StrongBlue}, #1
}

\usepackage{hyperref}
\hypersetup{
    colorlinks=true,
    citecolor=VibrantPink,
    linkcolor=VibrantPink,
    urlcolor=VibrantPink,
    pdftitle={Dual Advantage Fields},
    pdfauthor={Alexey Zemtsov, Maxim Bobrin, Alexander Nikulin, Dmitry V. Dylov, Fakhri Karray, Vladislav Kurenkov, Martin Takáč, Arip Asadulaev}
}
\usepackage[capitalize,noabbrev]{cleveref}
\crefname{assumption}{Assumption}{Assumptions}
\Crefname{assumption}{Assumption}{Assumptions}

\newcommand{\ie}{\textit{i.e.}}
\newcommand{\eg}{\textit{e.g.}}
\newcommand{\mstd}[2]{#1\,{\color{stdgray}\scriptsize$\pm#2$}}
\newcommand{\bestmstd}[2]{\textbf{#1}\,{\color{stdgray}\scriptsize$\pm#2$}}

\title{Dual Advantage Fields}

\author{
Alexey Zemtsov$^{2}$, ~~Maxim Bobrin$^{3}$, ~~Alexander Nikulin$^{1}$, ~~Dmitry V. Dylov$^{3}$, \\
\textbf{Fakhri Karray$^{4}$, ~~Vladislav Kurenkov$^{1,5}$, ~~Martin Takáč$^{4}$, ~~Arip Asadulaev$^{4}$} \\
{$^{1}$}dunnolab ~~{$^{2}$}ITMO ~~{$^{3}$}Computational Imaging Lab \\
{$^{4}$}MBZUAI ~~{$^{5}$}Innopolis University
}

\begin{document}

\maketitle

\begin{abstract}
Offline goal-conditioned reinforcement learning requires both long-horizon reachability estimates and local action comparisons. Dual goal representations provide value fields that capture global goal reachability, but they do not directly specify which action should be preferred at a given state. We propose Dual Advantage Fields (DAF), a policy-extraction method that turns a bilinear dual value model into a local advantage signal. Under bilinear dual parameterization, the goal embedding is the gradient of the value field with respect to the state representation, which directly specifies which action should be preferred.
DAF introduces an action-effect model that predicts the discounted direction of change in state representation space induced by actions and scores them by the similarity between this change and the goal direction. In the realizable case, this score equals the goal-conditioned Bellman advantage, yielding a standard local policy-improvement guarantee. On OGBench locomotion, manipulation, and puzzle tasks, DAF improves aggregate RLiable metrics and performs strongly in settings where locally correct actions differ from direct movement toward the final goal.

\end{abstract}

\begin{figure}[h!]
    \centering
    \includegraphics[width=0.92\linewidth]{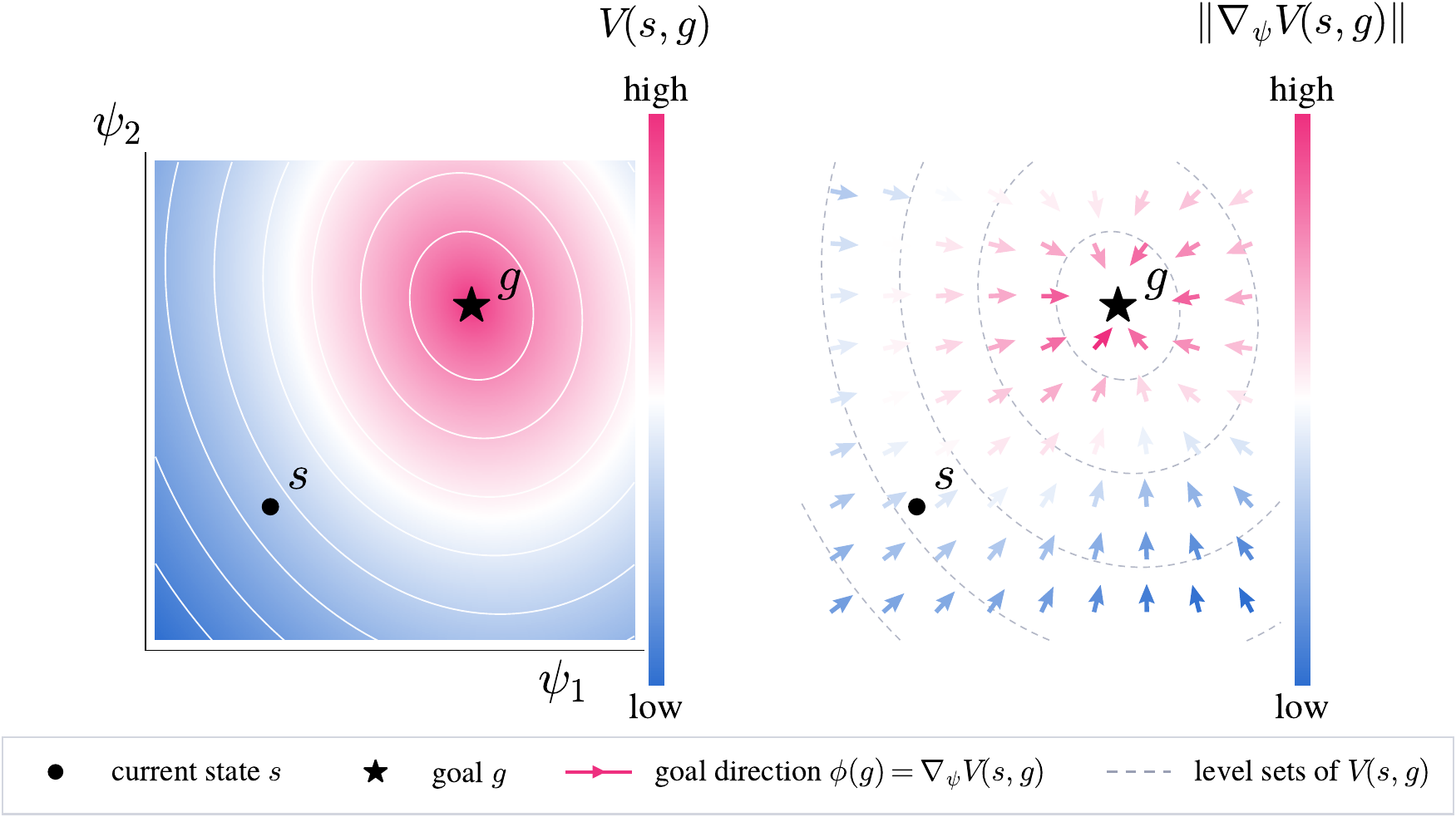}
    \vspace{-1mm}
    \caption{
    A dual goal-conditioned value model defines a global value surface over state
    representations. DAF uses the fact that the local action-comparative signal
    aligns with the goal direction.
    }
    \label{fig:teaser}
\end{figure}

\section{Introduction}
\label{sec:intro}

Goal-conditioned reinforcement learning (GCRL) aims to learn policies that reach
arbitrary goals from a fixed dataset of prior experience
\citep{liu2022goal,ma2022offline,park2024ogbench}. This requires solving two
different problems at once. First, the agent must reason globally: it must infer how
states are connected over long horizons so that behavior observed in one part of the
dataset can be stitched together with behavior observed elsewhere
\citep{park2023hiql,park2025horizon}. Second, the agent
must act locally: at the current state, it must decide which available action makes
the most progress toward the requested goal
\citep{peng2019advantage,kostrikov2021offline}. A good goal-conditioned agent therefore
needs both a global map of reachability and a local compass for action selection.

Most recently, dual goal representations \citep{park2025dual}
provide one of the strong answers to the first problem.
They parameterize a goal-conditioned value function as a bilinear interaction between
a state embedding and a goal embedding $V_\theta(s,g) = \psi_\theta(s)^\top \phi_\theta(g)$,
in the spirit of bilinear value networks \citep{hong2022bilinear}.
This structure induces a value surface for each goal, where states that are more
reachable or desirable for the goal receive higher values.
However, a value surface alone does not directly answer the local control question.

Two actions can start from the same state and therefore
share the same value \(V_\theta(s,g)\), while only one of them may move the agent
toward the goal. What is missing is
a local advantage-like quantity: a way to score whether an action changes the
state in a direction that improves goal-conditioned value. The usual remedy is to build a
second object for the purpose, such as an action-value function
\citep{kostrikov2021offline} or a subgoal predictor
\citep{park2023hiql}; notably, \citet{park2025dual} report that the dual field is
insufficient as the object a policy is extracted from, and therefore use dual
representations only as conditioners while training a separate value branch for extraction.

Our key observation is that this local signal is already implicit in the geometry of
dual representations, and that no additional machinery is needed to read it out.
Under the bilinear parameterization above, the goal embedding
\(\phi_\theta(g)\) is the direction in state-representation space along which the
goal-conditioned value increases.
This insight shows that the goal embedding is not merely a descriptor of a target
state: it is the dual object that evaluates local directions in representation space.
Consequently, combining this goal embedding with an action-effect vector yields an
advantage signal that is inherently local and action-comparative, the kind of
advantage that suffices for policy iteration.

We introduce \emph{Dual Advantage Fields} (DAF), a policy-extraction method that
makes this geometry explicit (Figure~\ref{fig:teaser}).
DAF learns an
action-effect model that predicts the discounted change in the state representation
caused by an action.
The resulting score is local,
goal-conditioned, and action-comparative.

Our contributions are:
\begin{itemize}[leftmargin=1.4em,itemsep=1pt,topsep=2pt]
    \item We show that, under the dual goal representation
    parameterization \citep{park2025dual}, the goal embedding is itself the gradient direction,
    in state representation space, of the goal-conditioned value field.

    \item We introduce \emph{Dual Advantage Fields}, an action-effect model that scores actions
    by their alignment with the direction toward the goal and therefore produces a local
    advantage-like signal for goal-conditioned policy extraction. We prove that this signal can
    be viewed as the goal-conditioned Bellman advantage, and hence as a valid improvement
    operator, when the dual model is realizable (Proposition~\ref{prop:daf_improvement}). We
    justify theoretically and empirically why the action effect should be learned rather than
    read off a single transition (Proposition~\ref{prop:amortization}).

    \item We use this signal to extract policies from offline data without training
    a separate goal-conditioned action-value function, and evaluate the resulting
    method on the $20$ state-based OGBench tasks \citep{park2024ogbench}, where it improves both
    per-task and aggregate RLiable performance \citep{rliable_agarwal2021} over flat and
    hierarchical backbones alike.
\end{itemize}

\section{Preliminaries}
\label{sec:prelim}
\paragraph{Goal-conditioned Reinforcement Learning.}
We study \emph{offline} goal-conditioned reinforcement learning (GCRL) \citep{liu2022goal,eysenbach2022contrastive,ma2022offline,myers2025offline}: the learner has access to a
fixed offline dataset of transitions but cannot collect new
experience in the environment~\citep{park2024ogbench}. The objective is to infer an optimal goal-conditioned policy even for state-goal pairs whose combination is unseen during training.

Let $\mathcal{S}$ and $\mathcal{A}$ denote state and action spaces and let
$\mathcal{G}\subseteq\mathcal{S}$ (or an abstract goal space) denote goals. At each step the
environment emits a transition $(s,a,s')$ according to an unknown Markov kernel
$P(s'\!\mid\!s,a)$, and a goal $g\in\mathcal{G}$ induces a reward $r(s,a,g)$, which in sparse
goal-reaching problems is often zero until a success condition holds. With discount
$\gamma\in(0,1)$, the goal-conditioned value functions of a policy $\pi(a\!\mid\!s,g)$ satisfy
\begin{equation}\label{eq:gcq}
  Q^\pi(s,a,g) = \mathbb{E}_{s'\sim P(\cdot\mid s,a)}\bigl[r(s,a,g)+\gamma\,V^\pi(s',g)\bigr],
  \quad
  V^\pi(s,g) = \mathbb{E}_{a\sim\pi(\cdot\mid s,g)}\bigl[Q^\pi(s,a,g)\bigr].
\end{equation}
Recent GCRL methods combine several ideas, including representation learning, quasimetric objectives \citep{wang2023optimal,ke2025hierarchical,myers2025offline}, and hierarchical
horizon reduction \citep{park2023hiql,giammarino2026goal,park2025horizon} over value functions, $Q$-functions, and actors. These design choices are often
complementary, but existing methods still show domain-specific strengths: hierarchical methods
tend to excel in long-horizon locomotion, while quasimetric representations often work well for
manipulation.

\textbf{Dual Goal Representations} \citep{park2025dual}.
The goal representation determines what information the policy and value function use about the
target state. Rather than conditioning on the raw goal observation, which may contain irrelevant or
exogenous factors, dual goal representations encode a goal by its reachability relation to other
states, in the same spirit in which a point may be described by its distances to the rest of the
space rather than by its coordinates. Formally, a goal $g$ is represented not by the observation
itself but by the function $\phi^\vee(g): s \mapsto d^\star(s,g)$, where the superscript $\vee$
marks this functional, \emph{dual} view of the goal and $d^\star(s,g)$ is the optimal discounted
temporal distance introduced below. In practice we approximate this functional through a bilinear
goal-conditioned potential \citep{hong2022bilinear}:
\begin{equation}
    V_\theta(s,g) = \psi_\theta(s)^\top \phi_\theta(g),
    \label{eq:bilinear-v}
\end{equation}
where $\psi_\theta:\mathcal{S}\to\mathbb{R}^d$ and
$\phi_\theta:\mathcal{G}\to\mathbb{R}^d$ are state and goal embeddings. The finite-dimensional
vector $\phi_\theta(g)$ is then the learned dual representation: paired with
$\psi_\theta(s)$, it predicts a value or distance-like quantity that reflects the environment's
reachability structure.

\textbf{Hierarchical Implicit Q-Learning (HIQL).} \label{prelim: hiql} \ \ In GCRL, accurately estimating the value function for distant goals is the main challenge in solving complex long-horizon tasks \citep{park2023hiql}.
To address this issue, HIQL \citep{park2023hiql} proposed a hierarchical policy structure that utilizes a value function learned with IQL \citep{kostrikov2021offline}.
This hierarchical design enables the agent to produce effective actions even when value estimates for distant goals are noisy or unreliable.
More specifically, HIQL trains a goal-conditioned state-value function $V$ with the following loss:
\begin{align}
    \mathcal{L}(V)=\mathbb{E}_{(s,s') \sim \mathcal{D},\; g \sim p(g)} \left[ L_2^{\tau} \left( r(s,g) + \gamma \bar{V}(s',g) - V(s,g) \right) \right], \label{eq:iql_value_loss}
\end{align}
where the expectile loss is defined as $L_2^\tau(u) = |\tau - \mathbf{1}(u < 0)|u^2$, with $\tau > 0.5$, and $\bar{V}$ denotes the target $V$ network.
Following prior works \citep{andrychowicz2017hindsight,wang2023optimal,park2023hiql}, we adopt the sparse reward $r(s,g)=-\mathbf{1} \{s \neq g \}$.
Under this reward, the optimal value $|V^{\star}(s,g)|$ corresponds to the \textit{discounted temporal distance} $d^\star(s,g)$, \ie, a discounted measure of the minimum number of environment steps required to reach the goal $g$ from state $s$.
HIQL separates policy extraction\footnote{Policy extraction refers to learning a policy from a learned value function, emphasizing the separation between value learning and policy learning.} into two levels: a high-level policy $\pi^{h}(s_{t+k} | s_t, g)$ generates a $k$-step subgoal to guide progress toward the goal, while a low-level policy $\pi^{\ell}(a_t | s_t, s_{t+k})$ produces primitive actions to reach the subgoal.
Both policies are extracted using advantage-weighted
regression (AWR) \citep{peng2019advantage,wang2020critic} with the following objective:
\begin{align}
    \mathcal{J}(\pi^h)&=\mathbb{E}_{(s_t,s_{t+k},g)\sim \mathcal{D}} \left[ \exp\left( \beta^h \cdot A^h(s_t,s_{t+k},g) \right) \log \pi^h(s_{t+k} | s_t,g) \right], \label{eq:high_actor_loss} \\
    \mathcal{J}(\pi^\ell)&=\mathbb{E}_{(s_t,a_t,s_{t+1}, s_{t+k})\sim \mathcal{D}} \left[ \exp\left( \beta^\ell \cdot A^{\ell}(s_t, s_{t+1}, s_{t+k}) \right) \log \pi^{\ell}(a_t |s_t,s_{t+k}) \right],\label{eq:low_actor_loss}
\end{align}
where $\beta^h$ and $\beta^l$ are inverse temperature parameters, $A^h (s_t, s_{t+k}, g)=V^h(s_{t+k}, g)-V^h(s_t,g)$ denotes the high-level policy advantage, and $A^{\ell}(s_t, s_{t+1}, s_{t+k})=V^\ell (s_{t+1},s_{t+k})-V^\ell (s_t,s_{t+k})$ denotes the low-level policy advantage.
HIQL uses a single goal-conditioned value function $V$, which is shared between both $\pi^h$ and $\pi^\ell$ (\ie, $V^h=V^\ell=V$).
However, despite this design, HIQL still struggles with long-horizon, complex tasks, as shown in the GCRL benchmark, OGBench \citep{park2024ogbench}.

\section{Dual Advantage Fields}
\label{sec:method}

\begin{figure}[t!]
    \centering
    \includegraphics[width=\linewidth]{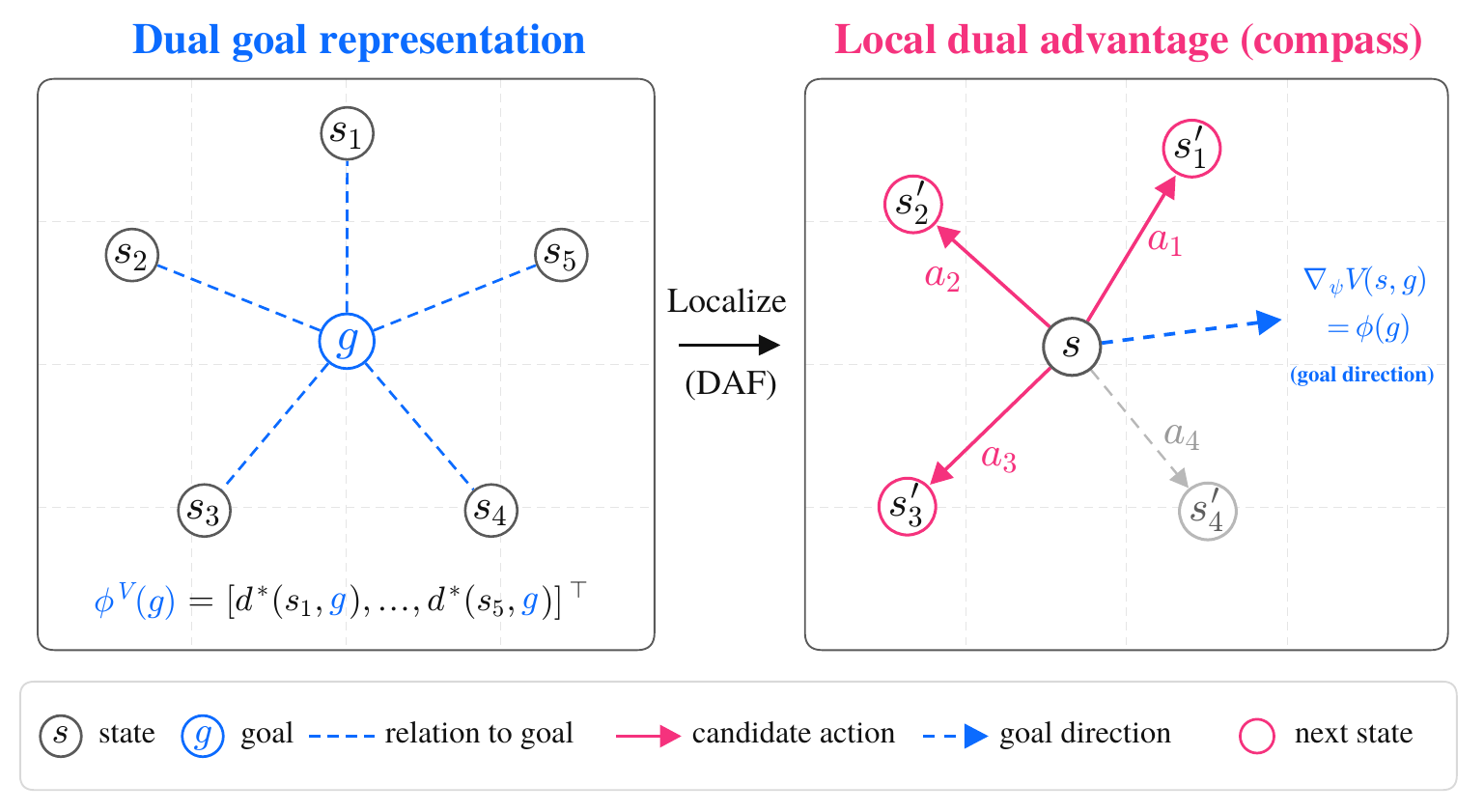}
    \caption{
    DAF scores an action by projecting its induced feature displacement onto the goal direction, yielding a local advantage signal for policy improvement.
    }
    \label{fig:fields}
\end{figure}

A dual goal representation encodes a goal by how it sits relative to other states.
The same vector can be read locally. With the goal held fixed, the bilinear model is a
scalar field on the state representation, and the direction of steepest increase is a
comparison among actions rather than a statement of long-horizon reachability. Dual
Advantage Fields make that comparison explicit.
Holding the goal fixed and viewing the value as a function of the state
embedding $\psi$, we have
\begin{proposition}
\label{prop:grad-identity}
Under the bilinear goal-conditioned value model
$V_\theta(s,g)=\psi_\theta(s)^\top\phi_\theta(g)$, the gradient of the value with respect to the
state embedding is the goal embedding:
\begin{equation}
    \nabla_{\psi} V_\theta(s,g)
    =
    \nabla_{\psi}\bigl(\psi^\top \phi_\theta(g)\bigr)
    =
    \phi_\theta(g).
    \label{eq:grad-psi}
\end{equation}
\end{proposition}
Thus, the goal embedding $\phi_\theta(g)$ is the value-gradient direction in representation space
under the Euclidean geometry of the learned embedding; see Figure~\ref{fig:fields} for intuition.
For any transition from $s$ to $s'$, the change in the bilinear value is exactly
\begin{equation}
    V_\theta(s',g)-V_\theta(s,g)
    =
    \phi_\theta(g)^\top
    \bigl(\psi_\theta(s')-\psi_\theta(s)\bigr).
    \label{eq:v-diff-bilinear}
\end{equation}

We use this identity to construct an advantage-like local policy improvement signal. For a policy
$\pi$, the standard goal-conditioned advantage is
\begin{equation}
    A^\pi(s,a,g)
    =
    \mathbb{E}_{s'\sim p(\cdot|s,a)}
    \left[
        r(s,a,g) + \gamma V^\pi(s',g) - V^\pi(s,g)
    \right].
    \label{eq:adv-bellman}
\end{equation}
Replacing $V^\pi$ with the learned bilinear field $V_\theta$ gives the model-induced Bellman
advantage
\begin{equation}
    A_\theta(s,a,g)
    =
    \mathbb{E}_{s'\sim p(\cdot|s,a)}
    \left[
        r(s,a,g)
        +
        \phi_\theta(g)^\top
        \bigl(\gamma\psi_\theta(s')-\psi_\theta(s)\bigr)
    \right].
    \label{eq:adv-dual-expected}
\end{equation}
In offline learning, each dataset transition $(s,a,s',g)$ provides a sample estimate of this quantity:
\begin{corollary}
\label{cor:sample-score}
The sample-level Dual Advantage Field score is
\begin{equation}
    \boxed{
    \widehat{A}_\theta(s,a,s',g)
    =
    r(s,a,g)
    +
    \phi_\theta(g)^\top
    \bigl(\gamma\psi_\theta(s')-\psi_\theta(s)\bigr).
    }
    \label{eq:adv-dual}
\end{equation}
\end{corollary}
Note that~\eqref{eq:adv-dual-expected} is an expectation over successors while~\eqref{eq:adv-dual}
is a single draw from it; \cref{sec:method_rep} replaces the draw by a learned module, for
theoretical and empirical reasons.

In~\eqref{eq:adv-dual}, the term $\gamma\psi_\theta(s')-\psi_\theta(s)$ is the discounted feature
displacement caused by action \(a\), and \(\phi_\theta(g)\) is the value-gradient direction
toward goal \(g\). Their inner product measures the one-step increase in the
bilinear value field, with the reward term completing the Bellman advantage, so
\(\widehat A_\theta\) is a local, goal-conditioned action-ranking
signal derived from the learned dual value geometry. This follows the
comparative view of policy improvement, where actions are improved by relative
advantages rather than absolute value estimates~\citep{dayan1995improving}.

\paragraph{Local policy improvement.}
In the realizable case, the DAF score is exactly the goal-conditioned Bellman
advantage: if \(V^\pi(s,g)=\psi(s)^\top\phi(g)\) and
\(u(s,a)=\mathbb{E}_{s'\sim P(\cdot\mid s,a)}[\gamma\psi(s')-\psi(s)]\), then
\(r(s,a,g)+u(s,a)^\top\phi(g)=A^\pi(s,a,g)\).
Thus, increasing the probability of actions with positive DAF score is a
standard policy-improvement step, and repeated exact DAF improvement
recovers an optimal goal-conditioned policy; in particular, its
limiting policy is at least as good as any policy restricted to a fixed
hierarchical class. We give the formal statement, the assumptions it needs, and the proof in
\cref{prop:daf_improvement}.

\section{Goal-conditioned policy extraction}
\label{sec:policy-extraction}

\Citet{dayan1995improving} argue that policy iteration should only require relative measures of
action utility, not absolute values. In their continuous deterministic setting they define an
advantage $A^{w}(x,u)=r(x,u)+f(x,u)\cdot\nabla_{x}V^{w}(x)$ and improve the policy by maximizing
$A^{w}$. They emphasize that learning these derivatives can be underconstrained, and introduce
conservative-field structure, a curl-free gradient, as an additional constraint.

We take this as intuition for extending their model-based comparison to goal-conditioned
control, with the empirical findings of provable performance in continuous scenario. The local improvement
step addresses the discounted change in the state representation,
$\gamma\psi_{\theta}(s')-\psi_{\theta}(s)$, which serves as a conceptual analogue to velocity: the Bellman
advantage~\eqref{eq:adv-bellman} is the inner product of that change with the goal
direction~\eqref{eq:adv-dual-expected}, in the spirit of $r+f\cdot\nabla V$.

We estimate the discounted increment
$\gamma\psi_\theta(s')-\psi_\theta(s)$ with a map $u_\xi(s,a)$ trained on offline transitions
(Sec.~\ref{sec:method_rep}).
The raw dual score is
\begin{equation}
  \label{eq:raw-dual}
  z_\theta(s,a,g) \;:=\; u_\xi(s,a)^\top \phi_\theta(g)\,,
\end{equation}
which agrees with~\eqref{eq:adv-dual} when $u_\xi(s,a)\approx \gamma\psi_\theta(s')-\psi_\theta(s)$
and rewards are handled by the value heads feeding the same Bellman targets.
Everything that
follows is standard offline-RL machinery assembled around this one scalar, and \cref{alg:daf}
summarizes the resulting training loop.

\label{sec:method_rep}

We learn $(\psi_\theta,\phi_\theta)$, and the
displacement map $u_\xi$ from offline tuples $(s,a,s',g)$~\citep{park2024ogbench}. For stability we used a common approach in offline RL~\citep{kostrikov2021offline} that learns twin critics $Q^{(1)}_\eta,Q^{(2)}_\eta$, and the bilinear
head $V_\theta(s,g)=\psi_\theta(s)^\top\phi_\theta(g)$ is tied to pessimistic $Q$-estimates via
expectile regression~\eqref{eq:iql_value_loss} and to Bellman backups on $Q^{(j)}_\eta$.

Fitting a value to $\max_a Q_\theta(s,a,g)$ is the step that most often destabilizes
continuous-control critics: the maximum over a continuous action set has to be approximated by an
actor, which introduces maximization bias~\citep{hasselt2016double,fujimoto2018addressing} and, in
deterministic sparse-reward problems, can deadlock actor and critic~\citep{matheron2019ddpg}.
DAF reads a \emph{direction}
out of the field,
so an error can rotate
$\phi_\theta(g)$ and invert the ranking of actions. We ground $u_\xi$ with the auxiliary loss
\begin{equation}
  \label{eq:ae-mse}
  \mathcal{L}_{\mathrm{ae}} \;=\; \mathbb{E}\Bigl[\bigl\|u_\xi(s,a) - \mathrm{sg}\bigl(\gamma\psi_\theta(s')-\psi_\theta(s)\bigr)\bigr\|_2^2\Bigr]\,,
\end{equation}
with $\mathrm{sg}$ stopping gradients through the target, so $u_\xi$ tracks one-step feature
dynamics on~$\mathcal{D}$.
\eqref{eq:adv-dual-expected}
asks for a conditional expectation over successors, and a dataset transition supplies one draw of
it.
By learning the action-effect network we amortize that expectation across the dataset instead of a single-sample draw. The theoretical comparison of these scenarios is given in \cref{prop:amortization}.

We therefore add an actor-free coupling between the value function $V_\theta$ and the dual score
$z_\theta$ of~\eqref{eq:raw-dual}, following \citet{perrin2024afu}. In continuous control the
usual critic backup maximizes a $Q$-function over actions, and that maximum has to be
approximated by an actor. The actor is what injects maximization bias and, in deterministic
sparse-reward problems, what can deadlock the critic~\citep{matheron2019ddpg}. Actor-free
updates avoid that path: they split the Bellman target into a state-value part and a
non-positive advantage part, and fit the two with an asymmetric piecewise loss. When the
backup is above the current estimate the value and the advantage are coupled; when it is
below, they are fit separately, so a pessimistic error in one cannot be cancelled by an
optimistic error in the other. The policy parameters receive no gradient from this loss, so
the critic is updated with no actor and no policy gradient in the backup
(Appendix~\ref{app:afu-details}).

\label{sec:method_awr}

Let $\pi_\omega(a\mid s,c)$ denote the policy, conditioned on $g$ through $\phi_\theta(g)$.
Advantage-weighted regression~\citep{peng2019advantage} uses weights
\begin{equation}
  w_\theta(s,a,g) \;=\; \min\Bigl\{\exp\bigl(\alpha\, z_\theta(s,a,g)\bigr),\,W_{\max}\Bigr\}
  \label{eq:awr-weights}
\end{equation}
with temperature $\alpha>0$ and cap $W_{\max}$, and minimizes
$-\mathbb{E}_{\mathcal{D}}[w_\theta \log \pi_\omega(a\mid s,c)]$. Because $z_\theta$ does
not depend on $\omega$, this is weighted behavior cloning that up-weights actions whose local
$\psi$-displacement aligns with the goal direction~$\phi_\theta(g)$.

\begin{algorithm}[t]
\caption{DAF training (with optional high-level actor).}
\label{alg:daf}
\begin{algorithmic}
\STATE \textbf{Input:} offline dataset $\mathcal{D}$ of tuples $(s,a,s',g)$; discount $\gamma$;
       expectile $\tau$; AWR temperature $\alpha$ and cap $W_{\max}$; subgoal stride $k$.
\STATE \textbf{Initialize:} dual factors $\psi_\theta,\phi_\theta$ with
       $V_\theta(s,g)=\psi_\theta(s)^\top\phi_\theta(g)$ \eqref{eq:bilinear-v}; twin critics
       $Q^{(j)}_\eta$; action-effect map $u_\xi$; policies $\pi_\omega,\pi^h_{\omega_h}$; targets.
\WHILE{not converged}
\STATE Sample a minibatch $\{(s,a,s',g)\}\sim\mathcal{D}$.
\STATE $\displaystyle\min_{\eta}\;\mathbb{E}\Bigl[\bigl(Q^{(j)}_\eta(s,a,g)-\bigl(r(s,g)+\gamma V^{\mathrm{tgt}}_\theta(s',g)\bigr)\bigr)^2\Bigr]$.
\STATE $\displaystyle\min_{\theta}\;\mathbb{E}\Bigl[L_2^{\tau}\Bigl(\min_j Q^{\mathrm{tgt},(j)}_\eta(s,a,g)-V_\theta(s,g)\Bigr)\Bigr]$
       \eqref{eq:iql_value_loss},
       where $L_2^{\tau}(u)=\lvert\tau-\mathbf{1}(u<0)\rvert u^2$.
\STATE $z_\theta(s,a,g)=u_\xi(s,a)^\top\phi_\theta(g)$ \eqref{eq:raw-dual}.
\STATE $\min\;\mathcal{L}_{\mathrm{afu}}$ \eqref{eq:afu-loss}
       \COMMENT{Appendix~\ref{app:afu-details}}.
\STATE $\displaystyle\min_{\xi}\;\mathbb{E}\Bigl[\bigl\|u_\xi(s,a)-\mathrm{sg}\bigl(\gamma\psi_\theta(s')-\psi_\theta(s)\bigr)\bigr\|_2^2\Bigr]$
       \eqref{eq:ae-mse}.
\STATE set $w_\theta\leftarrow\min\{\exp(\alpha z_\theta),W_{\max}\}$
       \eqref{eq:awr-weights} and minimize
       $-\mathbb{E}_{\mathcal{D}}[w_\theta\log\pi_\omega(a\mid s,c)]$ over $\omega$.
\STATE \textbf{(optional)} \begin{tabular}[t]{@{}l@{}}
       set $w^h_\theta\leftarrow\min\bigl\{\exp\bigl(\beta^h\bigl(V_\theta(s_{t+k},g)-V_\theta(s_t,g)\bigr)\bigr),W_{\max}\bigr\}$ and \\
       minimize $-\mathbb{E}[w^h_\theta\log\pi^h_{\omega_h}(s_{t+k}\mid s_t,g)]$ over $\omega_h$
       \eqref{eq:high_actor_loss}.
       \end{tabular}
\STATE Polyak-update the target networks.
\ENDWHILE
\STATE \textbf{Return} $\pi_\omega$ and $\pi^h_{\omega_h}$.
\end{algorithmic}
\end{algorithm}

\section{Experiments}
\label{sec:experiments}

\subsection{Localization Illustration}
\label{sec:motivational-example}

\begin{wrapfigure}[18]{r}{0.48\textwidth}
  \centering
  \captionsetup{justification=raggedright,singlelinecheck=false}
  \includegraphics[width=\linewidth]{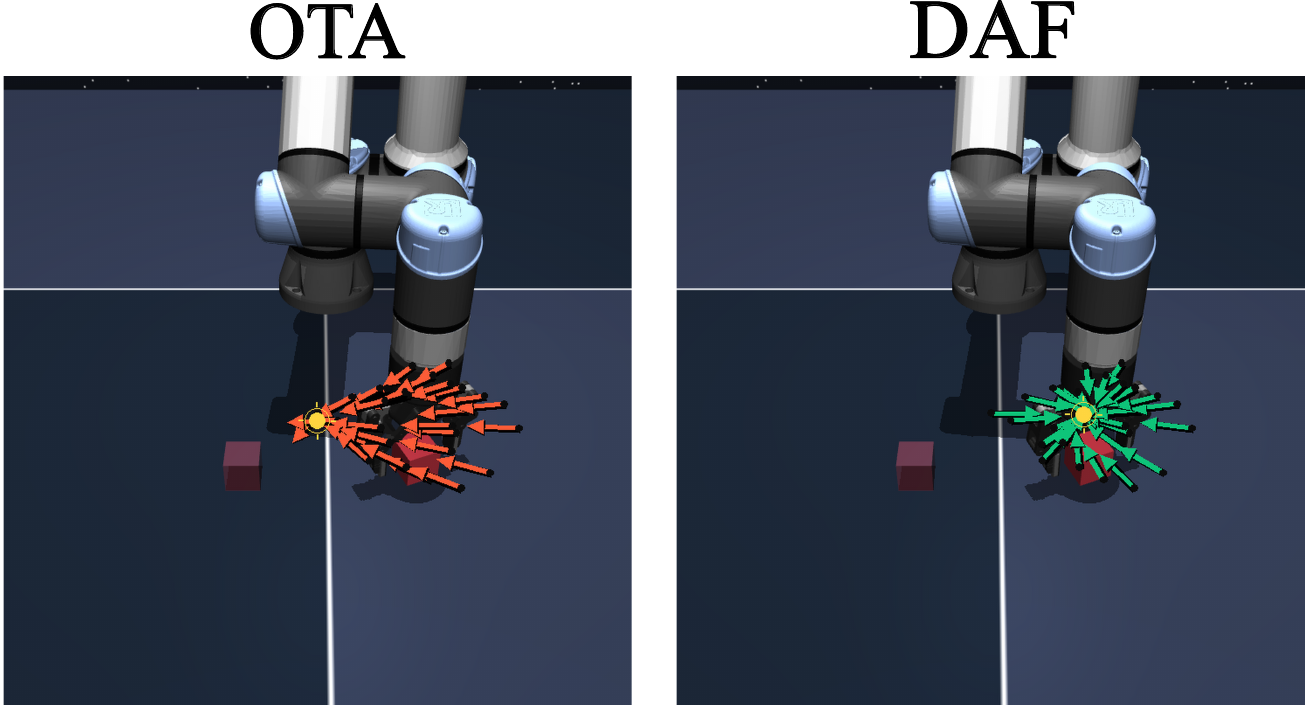}
  \caption{\textbf{Pre-grasp vector field in \texttt{cube-single}.} Arrows show decoded high-level directions from sampled gripper positions around the cube, with the cube and final goal fixed. DAF points locally toward the cube before grasping, while OTA points toward the terminal placement goal. The yellow marker denotes the mean decoded target.}
  \label{fig:toy_example}
\end{wrapfigure}

We illustrate the local geometry captured by Dual Advantage Fields on the
\texttt{cube-single-play-v0-task1} manipulation task from OGBench~\citep{park2024ogbench}, which
highlights a common failure mode in goal-conditioned control: before the cube can be
placed at the final target, the agent must first move the gripper into a pre-grasp configuration,
so a direction that points directly toward the terminal object location may be globally plausible
but locally unhelpful. We compare the resulting directions with option-aware temporally
abstracted value learning (OTA~\citep{ota2025option}), a state-of-the-art hierarchical method
and a standard backbone among strong goal-conditioned approaches.

The raw dual score~\eqref{eq:raw-dual} favors actions whose predicted feature displacement is aligned with the local direction of value increase toward the goal.

\needspace{0.28\textheight}
Figure~\ref{fig:toy_example} visualizes this effect. We sample query states
$\{\tilde{s}_i\}_{i=1}^N$ by perturbing only the gripper position around the cube, keeping the
object state and final goal fixed, and for each method $m \in \{\mathrm{OTA}, \mathrm{DAF}\}$
decode its high-level subgoal prediction into X-Y coordinates,
\begin{equation}
    \hat{x}_i^{\,m}
    =
    D_m\!\left(h_m(\tilde{s}_i,g)\right),
    \label{eq:motivation_decode}
\end{equation}
where $h_m$ is the method-specific latent output and $D_m$ is a linear probe fitted on demonstration
states. The plotted direction is
\begin{equation}
    d_i^{\,m}
    =
    \frac{\hat{x}_i^{\,m} - x_{\mathrm{ee}}(\tilde{s}_i)}
         {\|\hat{x}_i^{\,m} - x_{\mathrm{ee}}(\tilde{s}_i)\|_2},
    \label{eq:motivation_visual_field}
\end{equation}
drawn from the gripper position $x_{\mathrm{ee}}(\tilde{s}_i)$.
Near the cube, DAF produces directions that point toward the object, matching the pre-grasp
behavior required before transport, whereas OTA points toward the terminal placement region, which
is appropriate only after grasping. This is why local advantage fields can be more useful than a
high-level subgoal alone: they select actions by whether they locally improve the goal-conditioned
potential, precisely the comparison that a value level set cannot make.

\subsection{OGBench benchmark}

We empirically validate the findings developed above on the OGBench benchmark
\citep{park2024ogbench}.
We focus on the state-based locomotion and manipulation tasks,
which lets us test whether DAF improves consistently across domains with
different control structure.
All methods are trained purely offline and evaluated without additional environment interaction;
we report success-based performance in $[0,1]$ averaged over $5$ seeds. In maze-style locomotion,
\texttt{navigate} data is collected from noisy expert policies that traverse the environment,
while \texttt{stitch} data contains shorter trajectory segments and therefore requires composing
partial behaviors into longer goal-reaching solutions. In manipulation, \texttt{play} data
contains natural temporally correlated interactions generated by scripted policies, whereas
\texttt{noisy} data increases state-action coverage through less structured exploration noise,
making the offline data more suboptimal.
The temperature $\alpha$ is the only extraction hyperparameter; following
the OGBench and HIQL protocols we select it per environment family from $\{0.1,3.0,10.0\}$, and
report the full sweep in Appendix~\ref{app:repro}.
Expectile regression~\eqref{eq:iql_value_loss} with $\tau>0.5$ is optimistic and upper-biases the value under stochastic dynamics.
Following prior work, the derivation in \cref{sec:prelim} treats the dynamics as deterministic; the stochastic case is studied in \cref{app:stochastic}.

\textbf{Baselines.}
We compare against a representative set of recent methods for offline GCRL: HIQL
\citep{park2023hiql}, OTA \citep{ota2025option}, MQE \citep{zheng2025mqe}, CRL
\citep{eysenbach2022contrastive}, the goal-conditioned implicit $Q$- and $V$-learning baselines
GCIQL and GCIVL as implemented in OGBench \citep{kostrikov2021offline,park2024ogbench}, and,
where applicable, their dual-goal-representation variants \citep{park2025dual}. These cover the
main families used in offline GCRL, including horizon reduction \citep{park2025horizon} and
representation priors such as quasimetrics.

\providecommand{\tabpmstd}[1]{}
\renewcommand{\tabpmstd}[1]{{\color{gray}\fontsize{4.5}{5.2}\selectfont$#1$}}
\providecommand{\mstd}[2]{}
\renewcommand{\mstd}[2]{#1\,\tabpmstd{\pm#2}}
\providecommand{\bestmstd}[2]{}
\renewcommand{\bestmstd}[2]{\textbf{#1}\,\tabpmstd{\pm#2}}
\providecommand{\bestcell}[2]{}
\renewcommand{\bestcell}[2]{{\setlength{\fboxsep}{1pt}\colorbox{LightBlue!22}{\bestmstd{#1}{#2}}}}

\begin{table}[t]
  \centering
  \setlength{\tabcolsep}{3.4pt}
  \renewcommand{\arraystretch}{1.08}
  \caption{
  Results on the $20$ state-based OGBench tasks. The \texttt{navigate} and \texttt{play} datasets
  evaluate learning from coherent near-expert data, while \texttt{stitch} requires composing
  trajectory segments and \texttt{noisy} tests robustness under less structured exploration. Best
  values are highlighted in \textcolor{LightBlue}{blue}.
  }
  \label{tab:ogbench_main}
  \resizebox{\textwidth}{!}{
  \small
  \begin{tabular}{@{}>{\footnotesize}lccccccccc@{}}
    \toprule
    \multicolumn{1}{c}{\textbf{Task}}
    & \textbf{DAF} & \textbf{HIQL} & \textbf{OTA} & \textbf{MQE} & \textbf{CRL}
    & \begin{tabular}[c]{@{}c@{}}\textbf{CRL}\\[-2pt]\textbf{Dual}\end{tabular}
    & \textbf{GCIQL} & \textbf{GCIVL}
    & \begin{tabular}[c]{@{}c@{}}\textbf{GCIVL}\\[-2pt]\textbf{Dual}\end{tabular} \\
    \midrule
    \texttt{humanoidmaze-medium-navigate} & \bestmstd{0.93}{0.03} & \bestmstd{0.91}{0.01} & \bestcell{0.95}{0.01} & \mstd{0.49}{0.09} & \mstd{0.59}{0.03} & \mstd{0.62}{0.03} & \mstd{0.31}{0.04} & \mstd{0.31}{0.03} & \mstd{0.32}{0.03} \\
    \texttt{humanoidmaze-large-navigate}  & \mstd{0.66}{0.03} & \mstd{0.45}{0.04} & \bestcell{0.83}{0.03} & \mstd{0.20}{0.07} & \mstd{0.26}{0.03} & \mstd{0.21}{0.05} & \mstd{0.04}{0.01} & \mstd{0.05}{0.01} & \mstd{0.04}{0.01} \\
    \texttt{humanoidmaze-medium-stitch}   & \bestmstd{0.90}{0.04} & \mstd{0.86}{0.03} & \bestcell{0.92}{0.01} & \mstd{0.62}{0.09} & \mstd{0.53}{0.03} & \mstd{0.57}{0.01} & \mstd{0.15}{0.03} & \mstd{0.14}{0.01} & \mstd{0.20}{0.02} \\
    \texttt{humanoidmaze-large-stitch}    & \bestcell{0.48}{0.06} & \mstd{0.32}{0.04} & \mstd{0.43}{0.04} & \mstd{0.18}{0.03} & \mstd{0.11}{0.02} & \mstd{0.06}{0.03} & \mstd{0.02}{0.00} & \mstd{0.02}{0.01} & \mstd{0.02}{0.01} \\
    \midrule
    \texttt{antmaze-teleport-navigate}    & \mstd{0.51}{0.08} & \mstd{0.46}{0.03} & \mstd{0.53}{0.03} & \mstd{0.49}{0.04} & \bestcell{0.60}{0.01} & \bestmstd{0.57}{0.04} & \mstd{0.37}{0.02} & \mstd{0.44}{0.02} & \mstd{0.41}{0.02} \\
    \texttt{antmaze-medium-navigate}      & \bestcell{0.98}{0.01} & \bestmstd{0.96}{0.01} & \bestmstd{0.97}{0.01} & \mstd{0.88}{0.05} & \bestmstd{0.96}{0.01} & \bestmstd{0.95}{0.02} & \mstd{0.76}{0.06} & \mstd{0.72}{0.06} & \mstd{0.79}{0.03} \\
    \texttt{antmaze-teleport-stitch}      & \bestcell{0.50}{0.05} & \mstd{0.38}{0.02} & \mstd{0.38}{0.01} & \mstd{0.40}{0.03} & \mstd{0.23}{0.03} & \mstd{0.15}{0.04} & \mstd{0.23}{0.02} & \mstd{0.45}{0.05} & \mstd{0.39}{0.05} \\
    \texttt{antmaze-medium-stitch}        & \bestcell{0.97}{0.02} & \bestmstd{0.96}{0.02} & \bestmstd{0.95}{0.02} & \bestmstd{0.96}{0.01} & \mstd{0.52}{0.04} & \mstd{0.52}{0.07} & \mstd{0.39}{0.06} & \mstd{0.48}{0.03} & \mstd{0.52}{0.03} \\
    \midrule
    \texttt{cube-double-play}             & \mstd{0.41}{0.04} & \mstd{0.13}{0.01} & \mstd{0.05}{0.01} & \mstd{0.03}{0.00} & \mstd{0.16}{0.01} & \mstd{0.38}{0.06} & \mstd{0.35}{0.06} & \mstd{0.33}{0.05} & \bestcell{0.58}{0.04} \\
    \texttt{cube-triple-play}             & \bestcell{0.17}{0.03} & \mstd{0.05}{0.02} & \mstd{0.02}{0.00} & \mstd{0.01}{0.00} & \mstd{0.06}{0.02} & \mstd{0.05}{0.05} & \mstd{0.02}{0.01} & \mstd{0.01}{0.01} & \mstd{0.01}{0.00} \\
    \texttt{cube-quadruple-play}          & \bestcell{0.03}{0.01} & \mstd{0.00}{0.00} & \mstd{0.00}{0.00} & \mstd{0.00}{0.00} & \mstd{0.00}{0.00} & \mstd{0.00}{0.00} & \mstd{0.00}{0.00} & \mstd{0.00}{0.00} & \mstd{0.00}{0.00} \\
    \texttt{cube-double-noisy}            & \bestcell{0.33}{0.05} & \mstd{0.03}{0.01} & \mstd{0.05}{0.03} & \mstd{0.07}{0.01} & \mstd{0.04}{0.02} & \mstd{0.08}{0.02} & \mstd{0.24}{0.06} & \mstd{0.17}{0.04} & \mstd{0.26}{0.02} \\
    \texttt{cube-triple-noisy}            & \bestcell{0.23}{0.01} & \mstd{0.04}{0.01} & \mstd{0.01}{0.00} & \mstd{0.04}{0.02} & \mstd{0.03}{0.01} & \mstd{0.06}{0.02} & \mstd{0.05}{0.01} & \mstd{0.11}{0.02} & \mstd{0.09}{0.03} \\
    \texttt{cube-quadruple-noisy}         & \bestcell{0.02}{0.01} & \mstd{0.00}{0.00} & \mstd{0.00}{0.00} & \mstd{0.00}{0.00} & \mstd{0.00}{0.00} & \mstd{0.01}{0.01} & \mstd{0.00}{0.00} & \mstd{0.00}{0.00} & \mstd{0.00}{0.00} \\
    \midrule
    \texttt{scene-play}                   & \bestcell{0.81}{0.04} & \mstd{0.55}{0.09} & \mstd{0.34}{0.04} & \mstd{0.20}{0.03} & \mstd{0.29}{0.02} & \mstd{0.56}{0.06} & \mstd{0.53}{0.02} & \mstd{0.51}{0.05} & \bestmstd{0.78}{0.07} \\
    \texttt{scene-noisy}                  & \bestmstd{0.43}{0.03} & \mstd{0.27}{0.02} & \mstd{0.10}{0.02} & \mstd{0.07}{0.02} & \mstd{0.02}{0.01} & \mstd{0.06}{0.01} & \mstd{0.29}{0.02} & \mstd{0.31}{0.05} & \bestcell{0.45}{0.02} \\
    \midrule
    \texttt{puzzle-3x3-play}              & \mstd{0.74}{0.04} & \mstd{0.17}{0.03} & \mstd{0.64}{0.06} & \mstd{0.11}{0.00} & \mstd{0.07}{0.01} & \mstd{0.09}{0.03} & \bestcell{0.98}{0.02} & \mstd{0.08}{0.02} & \mstd{0.08}{0.01} \\
    \texttt{puzzle-4x4-play}              & \mstd{0.40}{0.05} & \mstd{0.17}{0.03} & \bestcell{0.53}{0.05} & \mstd{0.17}{0.03} & \mstd{0.02}{0.01} & \mstd{0.07}{0.02} & \mstd{0.31}{0.02} & \mstd{0.26}{0.02} & \mstd{0.30}{0.05} \\
    \texttt{puzzle-3x3-noisy}             & \bestcell{0.98}{0.04} & \mstd{0.70}{0.10} & \mstd{0.64}{0.13} & \mstd{0.03}{0.01} & \mstd{0.37}{0.05} & \mstd{0.42}{0.05} & \bestmstd{0.95}{0.01} & \mstd{0.44}{0.15} & \mstd{0.50}{0.22} \\
    \texttt{puzzle-4x4-noisy}             & \bestcell{0.47}{0.03} & \mstd{0.31}{0.07} & \mstd{0.01}{0.00} & \mstd{0.02}{0.01} & \mstd{0.00}{0.00} & \mstd{0.00}{0.00} & \mstd{0.33}{0.11} & \mstd{0.24}{0.02} & \mstd{0.25}{0.02} \\
    \bottomrule
  \end{tabular}
  }
\end{table}

\textbf{Maze locomotion: long-horizon navigation and stitching.}
We first evaluate on \texttt{humanoidmaze} and \texttt{antmaze}, the top block of
Table~\ref{tab:ogbench_main}. These environments isolate the long-horizon navigation aspect
of offline GCRL: the agent must reach target states from diverse initial states using only fixed
offline data, with \texttt{humanoidmaze} the harder of the two because it combines full-body
humanoid control with long-horizon goal reaching.
DAF is
competitive with the strongest horizon-reduction baselines on \texttt{navigate} datasets and
obtains the best results on the harder \texttt{stitch} cases.

\textbf{Object manipulation: local control from imperfect demonstrations.}
Next, we evaluate on \texttt{cube} and \texttt{scene}.
The largest margins are on \texttt{cube-triple} and
\texttt{scene}, precisely the settings in which many actions sit at a similar value level.

\textbf{Puzzle rearrangement: continuous control with combinatorial structure.}
Finally, we evaluate on \texttt{puzzle}, robotic versions of Lights Out in which pressing one
button changes the state of neighboring buttons, combining continuous control with combinatorial
generalization over discrete configurations; the \texttt{3x3} and \texttt{4x4} variants further
increase the configuration space.

\textbf{Hierarchical goals.}
For long horizons, a high-level policy over subgoals can be trained alongside the low-level stack,
with value differences along options as in hierarchical offline GCRL~\citep{park2023hiql};
option-aware temporally abstracted value learning offers a related hierarchical
baseline~\citep{ota2025option}. In our main experiments the high level is extracted with the
value-gap objective~\eqref{eq:high_actor_loss} and the low level replaces the value gap
in~\eqref{eq:low_actor_loss} by $z_\theta$. The hierarchy and the dual advantage address different
bottlenecks: the former improves how value is \emph{propagated} over long horizons, the latter how
actions are \emph{compared} once a value field is available. Because $z_\theta$ consumes only a
bilinear field and a set of transitions, it is agnostic to which backbone produces the field, and
Appendix~\ref{app:repro} shows the same gain when it is attached to a flat critic instead.
Concrete instantiations differ by which of the optional terms above are active; experimental
details are in Appendix~\ref{app:repro}.

Table~\ref{tab:ogbench_main} shows that DAF performs strongly across different kinds of offline
coverage and control structure.
Figure~\ref{fig:rliable} summarizes the same runs with RLiable metrics
\citep{rliable_agarwal2021}: the interquartile mean reduces sensitivity to outlier tasks and the
optimality gap measures the average shortfall from perfect success. DAF improves all four
aggregate metrics, the gain is not driven by a single environment family, and the performance
profile (Figure~\ref{fig:myimage}) shows the same ordering at every score threshold.

\begin{figure}[!t]
    \centering
    \includegraphics[width=\linewidth]{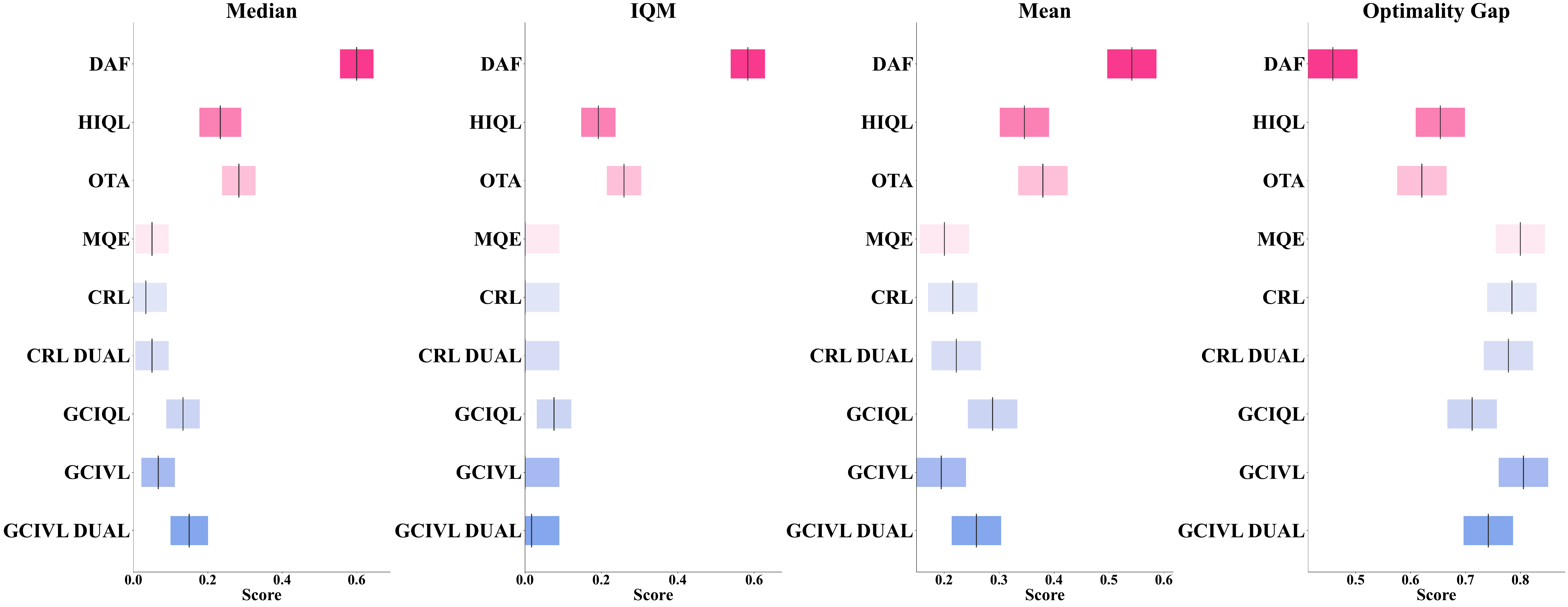}
    \caption{
    \textbf{Performance comparison.} Following \citet{rliable_agarwal2021}, we report Median, IQM,
    Mean and Optimality Gap with stratified-bootstrap confidence intervals; colored segments denote
    confidence intervals and dark markers point estimates.
    }
    \label{fig:rliable}
\end{figure}

\textbf{Ablations.}
Appendix~\ref{app:repro} removes each design choice in turn
(Table~\ref{tab:ogbench_ablation}).
Dropping the
actor-free coupling, the learned action effect or the hierarchy costs $0.33$, $0.28$ and $0.29$
success on average. The sharpest comparison is the last column: the same representations and
hierarchy, extracted conventionally, lose $0.45$. Conversely, the dual advantage on a flat critic
already beats the flat goal-conditioned baselines on most tasks (Table~\ref{tab:flat_vs_flat}), so
it improves extraction from either backbone.

\textbf{Extraction from a directly learned dual critic.}
The gradient identity says that the local signal is already the goal direction of the state--goal
field, so a separate action-conditioned critic is unnecessary. The natural alternative is to set
that identity aside and learn a bilinear critic $Q(s,a,g)=\psi(s,a)^\top\phi(g)$ directly, then
extract with $Q$ or with $Q-V$. This keeps the dual factorization and drops only the claim that
the state--goal field already contains the action comparison. Table~\ref{tab:dualq} compares the
two under the same training budget, hyperparameters, evaluation protocol and seeds, on nine tasks.
We report the better of the two extraction variants, which was $Q-V$ on every task. DAF is ahead on
all nine, by $0.07$--$0.41$. The gap is consistent with the observation that bilinear value models
degrade as the interaction they represent becomes more complex~\citep{hong2022bilinear}: letting
the state factor depend on the action strictly increases that complexity. It is also consistent
with the finding that a dual representation alone is not an object conventional extraction can act
on~\citep{park2025dual}, which is why that line of work adds a separate value branch. The gradient
identity removes the need for either. The state--goal field is enough: one reads a direction out of
it rather than a number.

\begin{table}[H]
  \centering
  \newcommand{\abcell}[1]{{\setlength{\fboxsep}{1pt}\colorbox{LightBlue!22}{\textbf{#1}}}}
  \caption{DAF versus extraction from a directly learned bilinear
  $Q(s,a,g)=\psi(s,a)^\top\phi(g)$, under an identical training and evaluation protocol with $5$
  seeds. Best values are highlighted in \textcolor{LightBlue}{blue}.}
  \label{tab:dualq}
  \begin{small}
  \setlength{\tabcolsep}{6pt}
  \begin{tabular}{@{}lcc@{}}
    \toprule
    \textbf{Task} & \textbf{DAF} & \textbf{Direct dual-$Q$} \\
    \midrule
    \texttt{scene-play}                 & \abcell{\mstd{0.81}{0.04}} & \mstd{0.63}{0.03} \\
    \texttt{cube-triple-play}           & \abcell{\mstd{0.17}{0.03}} & \mstd{0.10}{0.03} \\
    \texttt{cube-double-noisy}          & \abcell{\mstd{0.33}{0.05}} & \mstd{0.25}{0.02} \\
    \texttt{puzzle-4x4-noisy}           & \abcell{\mstd{0.47}{0.03}} & \mstd{0.06}{0.05} \\
    \texttt{humanoidmaze-large-stitch}  & \abcell{\mstd{0.48}{0.06}} & \mstd{0.24}{0.06} \\
    \texttt{antmaze-large-navigate}     & \abcell{\mstd{0.93}{0.02}} & \mstd{0.78}{0.02} \\
    \texttt{antmaze-large-stitch}       & \abcell{\mstd{0.56}{0.04}} & \mstd{0.33}{0.04} \\
    \texttt{antmaze-teleport-stitch}    & \abcell{\mstd{0.50}{0.05}} & \mstd{0.36}{0.03} \\
    \texttt{antmaze-teleport-navigate}  & \abcell{\mstd{0.51}{0.08}} & \mstd{0.42}{0.04} \\
    \bottomrule
  \end{tabular}
  \end{small}
\end{table}

\section{Broader Impact and Limitations}
\label{sec:broader_impact}

DAF extracts goal-conditioned policies from offline data without additional
environment interaction. Like other offline RL methods, it is reliable only when
the dataset sufficiently covers the actions needed for improvement; poor coverage
can produce incorrect action rankings. DAF also relies on the learned dual representation and action-effect model.

Although \(\nabla_{\psi}V_\theta(s,g)=\phi_\theta(g)\) holds exactly for the
bilinear head, this direction is useful only if the representation encodes
reachability. In stochastic or poorly covered regions, \(u_\xi(s,a)\) may predict
inaccurate feature displacements. Future work should study uncertainty-aware or
distributional action-effect models and extend DAF to image-based and more
stochastic goal-reaching settings.

\section{Conclusion}
\label{sec:conclusion}

We introduced Dual Advantage Fields, a method that turns dual goal representations into local
policy-improvement signals. Under the bilinear value parameterization the goal embedding acts as
the gradient of the goal-conditioned value field with respect to the state representation, and DAF
uses this to score actions by the alignment between their predicted feature displacement and the
goal direction. The resulting rule is a single inner product, needs no additional action-value
function, and in the realizable case is provably the goal-conditioned Bellman advantage.
Empirically it improves aggregate performance across offline GCRL benchmarks, is especially
effective in manipulation tasks where local directional choices matter, and carries over to flat
and hierarchical backbones alike. Dual representations should therefore be used not only as global
value maps, but also as local advantage fields for goal-conditioned policy extraction.

\section*{AI use statement}

We used generative AI tools, in particular a large language model (LLM), as a writing and coding
assistant. Concretely, we used them to improve the clarity and grammar of the exposition and to
produce the plotting scripts. We have reviewed all AI-assisted text and code and take full
responsibility for the final content of this work.

\section*{Ethics statement}

This work studies goal-reaching policies learned from fixed datasets in simulated robotic
benchmarks. It involves no human subjects, no personal data and no newly released dataset. The
main foreseeable risk is the generic one for offline RL: a policy extracted from data that does
not cover the behavior required for a task can rank actions confidently and incorrectly, and DAF
provides no uncertainty estimate that would flag this. We discuss this in
\cref{sec:broader_impact} and consider uncertainty-aware action-effect models an important
direction before any deployment on physical systems.

\section*{Reproducibility statement}

Our method and baselines are implemented on top of the public OGBench
\citep{park2024ogbench} and Dual Goal Representations \citep{park2025dual} codebases, and all
results use the unmodified OGBench datasets and evaluation protocol.
\Cref{alg:daf} gives the full training loop with a cross-reference to the defining equation for
every line; \cref{app:afu-details} gives the explicit form of the actor-free coupling.
\Cref{tab:config} lists the hyperparameters shared by all runs and
\cref{tab:alpha} the single per-environment hyperparameter, with the grid search that produced it
in \cref{tab:alpha_sweep}. \Cref{app:experiments-details} specifies the environments, dataset
variants, seeds and the exact definitions of the aggregate metrics and bootstrap procedure, and
\cref{app:compute} reports the compute used. All assumptions behind the theoretical claims are
stated formally in \cref{asm:realizable,asm:lineworld,asm:hier-indep} and each proposition in
\cref{app:theory-daf} is proved there. Source code is provided as anonymous supplementary
material.

\bibliography{bib/references}
\bibliographystyle{iclr2027_conference}

\clearpage
\appendix
\section{Implementation and Reproducibility}
\label{app:repro}

Our method and baselines are implemented on top of the implementations given in OGBench
\citep{park2024ogbench} and Dual Goal Representations \citep{park2025dual} codebases. Our method
is employed upon a hierarchy of actors, with the low-level actor being updated by the dual score
(\cref{eq:raw-dual}) and the high-level actor by AWR (\cref{eq:high_actor_loss}).

\cref{tab:config} details the common hyperparameters for all methods on OGBench.
\cref{tab:alpha} shows the $\alpha$ regularization hyperparameter that was found to be the best
for performance of DAF, and \cref{tab:alpha_sweep} the grid search behind it. We also report the
ablation studies on important architectural aspects of our proposed method: AFU coupling
\cite{perrin2024afu}, presence of the action-effect module (\cref{eq:ae-mse}), hierarchical actor
\cite{park2023hiql} and integrating dual representations \cite{park2025dual} upon the
hierarchical backbone.

\begin{table}[H]
    \centering
    \caption{Network configuration for DAF on OGBench.}
    \label{tab:config}
    \begin{tabular}{rc}
        \toprule
        \textbf{Configuration}                     & \textbf{Value} \\
        \midrule
        Gradient steps                             & $10^6$ \\
        Optimizer                                  & Adam~\citep{adam_kingma2015} \\
        Nonlinearity                               & GELU~\citep{gelu_hendrycks2016} \\
        Target network update rate                 & $0.005$ \\
        Goal representation dimensionality $d$     & $256$ \\
        Batch size                                 & $1024$ \\
        Action-effect MLP dimensions               & $(512, 512, 512)$ \\
        Policy MLP dimensions                      & $(512, 512, 512)$ \\
        Layer norm in MLPs                         & \texttt{True} \\
        Discount ($\gamma$)                        & $0.99$ ($0.995$ for \texttt{giant-} environments) \\
        Learning rate                              & $0.0003$ \\
        Seeds per task                              & $5$ \\
        \bottomrule
    \end{tabular}
\end{table}

\begin{table}[H]
    \centering
    \caption{Coefficient $\alpha$ for each environment.}
    \label{tab:alpha}
    \newcolumntype{R}{>{\small}r}
    \begin{tabular}{Rc}
        \toprule
        \textbf{Environment} & \textbf{$\alpha$} \\
        \midrule
        \texttt{scene-*}                   & 10.0 \\
        \texttt{antmaze}                   & 10.0 \\
        \texttt{humanoidmaze}              & 10.0 \\
        \texttt{puzzle-3x3-*}              & 3.0  \\
        \texttt{puzzle-4x4-*}              & 0.1 \\
        \texttt{cube-double-*}             & 3.0  \\
        \texttt{cube-triple-*}             & 3.0  \\
        \texttt{cube-quadruple-*}          & 10.0 \\
        \bottomrule
    \end{tabular}
\end{table}

\paragraph{Sensitivity to the extraction temperature.}
The AWR temperature $\alpha$ of \cref{eq:awr-weights} is the single hyperparameter we select per
environment family, following the protocol of OGBench \citep{park2024ogbench} and HIQL
\citep{park2023hiql}, which both identify it as one of the most important hyperparameters in this
benchmark and publish per-environment values; \citet{park2025dual} follow the same convention.
\cref{tab:alpha_sweep} reports the full grid over $\{0.1, 3.0, 10.0\}$ with the training budget,
evaluation protocol and seeds held identical to the main tables. The ordering is genuinely
environment-dependent - $\alpha=10$ for mazes and \texttt{scene}, $3$ for \texttt{cube} and
\texttt{puzzle-3x3}, $0.1$ for \texttt{puzzle-4x4} - and moving to the selected value changes
success by $0.22$-$0.36$ on the maze and \texttt{scene} tasks, which is why it is tuned per
family rather than globally.

\begin{table}[H]
  \centering
  \newcommand{\abcell}[1]{{\setlength{\fboxsep}{1pt}\colorbox{LightBlue!22}{\textbf{#1}}}}
  \caption{Sensitivity of DAF to the AWR temperature $\alpha$. Selected settings
  (\cref{tab:alpha}) are highlighted in blue.}
  \label{tab:alpha_sweep}
  \begin{small}
  \begin{tabular}{@{}lccc@{}}
    \toprule
    \textbf{Task} & $\boldsymbol{\alpha=0.1}$ & $\boldsymbol{\alpha=3.0}$ & $\boldsymbol{\alpha=10.0}$ \\
    \midrule
    \texttt{humanoidmaze-medium-navigate} & \mstd{0.68}{0.05} & \mstd{0.85}{0.04} & \abcell{\mstd{0.93}{0.03}} \\
    \texttt{antmaze-medium-stitch}        & \mstd{0.61}{0.06} & \mstd{0.87}{0.04} & \abcell{\mstd{0.97}{0.02}} \\
    \texttt{antmaze-teleport-stitch}      & \mstd{0.28}{0.06} & \mstd{0.42}{0.06} & \abcell{\mstd{0.50}{0.05}} \\
    \texttt{scene-play}                   & \mstd{0.54}{0.06} & \mstd{0.73}{0.05} & \abcell{\mstd{0.81}{0.04}} \\
    \texttt{cube-double-noisy}            & \mstd{0.12}{0.04} & \abcell{\mstd{0.33}{0.05}} & \mstd{0.28}{0.05} \\
    \texttt{cube-triple-play}             & \mstd{0.09}{0.04} & \abcell{\mstd{0.17}{0.03}} & \mstd{0.12}{0.03} \\
    \texttt{puzzle-3x3-play}              & \mstd{0.54}{0.05} & \abcell{\mstd{0.74}{0.04}} & \mstd{0.63}{0.05} \\
    \texttt{puzzle-4x4-noisy}             & \abcell{\mstd{0.47}{0.03}} & \mstd{0.28}{0.04} & \mstd{0.22}{0.03} \\
    \bottomrule
  \end{tabular}
  \end{small}
\end{table}

\begin{table}[H]
  \centering
  \newcommand{\abcell}[1]{{\setlength{\fboxsep}{1pt}\colorbox{LightBlue!22}{\textbf{#1}}}}
  \setlength{\tabcolsep}{3.4pt}
  \renewcommand{\arraystretch}{1.08}
  \caption{Ablation on OGBench. We report full DAF and four ablations: removing AFU coupling,
  removing the action-effect model (using the direct one-sample value-difference score of
  \cref{eq:adv-dual}), removing hierarchy, and a dual-representation hierarchical baseline that
  extracts conventionally rather than with the dual advantage. Best values are highlighted in
  \textcolor{LightBlue}{blue}.}
  \label{tab:ogbench_ablation}
  \resizebox{\textwidth}{!}{
  \small
  \begin{tabular}{@{}>{\footnotesize}lccccc@{}}
    \toprule
    \multicolumn{1}{c}{\textbf{Task}}
    & \textbf{DAF}
    & \begin{tabular}[c]{@{}c@{}}\textbf{No AFU}\\[-2pt]\textbf{Coupling}\end{tabular}
    & \begin{tabular}[c]{@{}c@{}}\textbf{No Action-}\\[-2pt]\textbf{Effect}\end{tabular}
    & \begin{tabular}[c]{@{}c@{}}\textbf{DAF w/o}\\[-2pt]\textbf{Hierarchy}\end{tabular}
    & \begin{tabular}[c]{@{}c@{}}\textbf{Dual-Rep}\\[-2pt]\textbf{+ Hierarchy}\end{tabular} \\
    \midrule
    \texttt{humanoidmaze-medium-navigate} & \abcell{\mstd{0.93}{0.03}} & \mstd{0.18}{0.03} & \mstd{0.35}{0.03} & \mstd{0.38}{0.04} & \mstd{0.06}{0.01} \\
    \texttt{humanoidmaze-large-navigate}  & \abcell{\mstd{0.66}{0.03}} & \mstd{0.02}{0.01} & \mstd{0.04}{0.01} & \mstd{0.39}{0.05} & \mstd{0.01}{0.00} \\
    \texttt{humanoidmaze-medium-stitch}   & \abcell{\mstd{0.90}{0.04}} & \mstd{0.47}{0.06} & \mstd{0.50}{0.07} & \mstd{0.32}{0.06} & \mstd{0.05}{0.02} \\
    \texttt{humanoidmaze-large-stitch}    & \mstd{0.48}{0.06} & \mstd{0.06}{0.02} & \mstd{0.08}{0.01} & \abcell{\mstd{0.65}{0.07}} & \mstd{0.01}{0.01} \\
    \midrule
    \texttt{antmaze-teleport-navigate}    & \abcell{\mstd{0.51}{0.08}} & \mstd{0.33}{0.05} & \mstd{0.35}{0.05} & \mstd{0.39}{0.05} & \mstd{0.18}{0.07} \\
    \texttt{antmaze-medium-navigate}      & \abcell{\mstd{0.98}{0.01}} & \mstd{0.78}{0.05} & \mstd{0.93}{0.05} & \mstd{0.19}{0.04} & \mstd{0.89}{0.02} \\
    \texttt{antmaze-teleport-stitch}      & \abcell{\mstd{0.50}{0.05}} & \mstd{0.19}{0.04} & \mstd{0.17}{0.04} & \mstd{0.00}{0.01} & \mstd{0.09}{0.02} \\
    \texttt{antmaze-medium-stitch}        & \abcell{\mstd{0.97}{0.02}} & \mstd{0.42}{0.05} & \mstd{0.42}{0.10} & \mstd{0.00}{0.01} & \mstd{0.21}{0.06} \\
    \midrule
    \texttt{cube-double-play}             & \mstd{0.41}{0.04} & \mstd{0.36}{0.07} & \abcell{\mstd{0.51}{0.05}} & \mstd{0.39}{0.05} & \mstd{0.02}{0.01} \\
    \texttt{cube-triple-play}             & \abcell{\mstd{0.17}{0.03}} & \mstd{0.02}{0.01} & \mstd{0.04}{0.02} & \mstd{0.07}{0.02} & \mstd{0.01}{0.01} \\
    \texttt{cube-quadruple-play}          & \abcell{\mstd{0.03}{0.01}} & \mstd{0.00}{0.00} & \mstd{0.00}{0.00} & \mstd{0.01}{0.01} & \mstd{0.00}{0.00} \\
    \texttt{cube-double-noisy}            & \mstd{0.33}{0.05} & \mstd{0.26}{0.05} & \abcell{\mstd{0.39}{0.03}} & \mstd{0.35}{0.04} & \mstd{0.03}{0.01} \\
    \texttt{cube-triple-noisy}            & \abcell{\mstd{0.23}{0.01}} & \mstd{0.01}{0.01} & \mstd{0.05}{0.02} & \mstd{0.02}{0.01} & \mstd{0.01}{0.00} \\
    \texttt{cube-quadruple-noisy}         & \abcell{\mstd{0.02}{0.01}} & \mstd{0.00}{0.00} & \mstd{0.00}{0.00} & \mstd{0.01}{0.00} & \mstd{0.00}{0.00} \\
    \midrule
    \texttt{scene-play}                   & \abcell{\mstd{0.81}{0.04}} & \mstd{0.49}{0.05} & \mstd{0.52}{0.01} & \mstd{0.45}{0.04} & \mstd{0.20}{0.07} \\
    \texttt{scene-noisy}                  & \abcell{\mstd{0.43}{0.03}} & \mstd{0.29}{0.03} & \mstd{0.40}{0.04} & \mstd{0.37}{0.05} & \mstd{0.05}{0.02} \\
    \midrule
    \texttt{puzzle-3x3-play}              & \abcell{\mstd{0.74}{0.04}} & \mstd{0.10}{0.01} & \mstd{0.15}{0.01} & \mstd{0.02}{0.01} & \mstd{0.07}{0.02} \\
    \texttt{puzzle-4x4-play}              & \mstd{0.40}{0.05} & \mstd{0.18}{0.03} & \mstd{0.17}{0.03} & \abcell{\mstd{0.75}{0.06}} & \mstd{0.01}{0.00} \\
    \texttt{puzzle-3x3-noisy}             & \abcell{\mstd{0.98}{0.04}} & \mstd{0.15}{0.03} & \mstd{0.17}{0.02} & \mstd{0.37}{0.06} & \mstd{0.05}{0.01} \\
    \texttt{puzzle-4x4-noisy}             & \abcell{\mstd{0.47}{0.03}} & \mstd{0.07}{0.01} & \mstd{0.09}{0.01} & \mstd{0.03}{0.01} & \mstd{0.00}{0.01} \\
    \bottomrule
  \end{tabular}
  }
\end{table}

\paragraph{The dual advantage is backbone-agnostic.}
The same extraction applies whether the value field comes from a flat backbone or from a
hierarchy. Adding the dual advantage to the flat backbone improves over the flat goal-conditioned
methods, and adding it to the hierarchy improves over extracting from that hierarchy in the usual
way (\cref{tab:flat_vs_flat,tab:ogbench_ablation}). In both cases the method improves what can be
extracted; the hierarchy remains a separate choice about how value is propagated over long horizons.

\begin{table}[H]
  \centering
  \newcommand{\abcell}[1]{{\setlength{\fboxsep}{1pt}\colorbox{LightBlue!22}{\textbf{#1}}}}
  \caption{The dual advantage on a flat backbone, against the flat goal-conditioned methods of the
  main tables. Best values are highlighted in \textcolor{LightBlue}{blue}.}
  \label{tab:flat_vs_flat}
  \begin{small}
  \setlength{\tabcolsep}{5pt}
  \begin{tabular}{@{}lccc@{}}
    \toprule
    \textbf{Task} & \textbf{Flat DAF} & \textbf{GCIQL} & \textbf{GCIVL} \\
    \midrule
    \texttt{humanoidmaze-medium-navigate} & \abcell{\mstd{0.38}{0.04}} & \mstd{0.31}{0.04} & \mstd{0.31}{0.03} \\
    \texttt{humanoidmaze-large-navigate}  & \abcell{\mstd{0.39}{0.05}} & \mstd{0.04}{0.01} & \mstd{0.05}{0.01} \\
    \texttt{humanoidmaze-medium-stitch}   & \abcell{\mstd{0.32}{0.06}} & \mstd{0.15}{0.03} & \mstd{0.14}{0.01} \\
    \texttt{humanoidmaze-large-stitch}    & \abcell{\mstd{0.65}{0.07}} & \mstd{0.02}{0.00} & \mstd{0.02}{0.01} \\
    \texttt{antmaze-teleport-navigate}    & \mstd{0.39}{0.05} & \mstd{0.37}{0.02} & \abcell{\mstd{0.44}{0.02}} \\
    \texttt{antmaze-medium-navigate}      & \mstd{0.19}{0.04} & \abcell{\mstd{0.76}{0.06}} & \mstd{0.72}{0.06} \\
    \texttt{antmaze-teleport-stitch}      & \mstd{0.00}{0.01} & \mstd{0.23}{0.02} & \abcell{\mstd{0.45}{0.05}} \\
    \texttt{antmaze-medium-stitch}        & \mstd{0.00}{0.01} & \mstd{0.39}{0.06} & \abcell{\mstd{0.48}{0.03}} \\
    \texttt{cube-double-play}             & \abcell{\mstd{0.39}{0.05}} & \mstd{0.35}{0.06} & \mstd{0.33}{0.05} \\
    \texttt{cube-triple-play}             & \abcell{\mstd{0.07}{0.02}} & \mstd{0.02}{0.01} & \mstd{0.01}{0.01} \\
    \texttt{cube-quadruple-play}          & \abcell{\mstd{0.01}{0.01}} & \mstd{0.00}{0.00} & \mstd{0.00}{0.00} \\
    \texttt{cube-double-noisy}            & \abcell{\mstd{0.35}{0.04}} & \mstd{0.24}{0.06} & \mstd{0.17}{0.04} \\
    \texttt{cube-triple-noisy}            & \mstd{0.02}{0.01} & \mstd{0.05}{0.01} & \abcell{\mstd{0.11}{0.02}} \\
    \texttt{cube-quadruple-noisy}         & \abcell{\mstd{0.01}{0.00}} & \mstd{0.00}{0.00} & \mstd{0.00}{0.00} \\
    \texttt{scene-play}                   & \mstd{0.45}{0.04} & \abcell{\mstd{0.53}{0.02}} & \mstd{0.51}{0.05} \\
    \texttt{scene-noisy}                  & \abcell{\mstd{0.37}{0.05}} & \mstd{0.29}{0.02} & \mstd{0.31}{0.05} \\
    \texttt{puzzle-3x3-play}              & \mstd{0.02}{0.01} & \abcell{\mstd{0.98}{0.02}} & \mstd{0.08}{0.02} \\
    \texttt{puzzle-4x4-play}              & \abcell{\mstd{0.75}{0.06}} & \mstd{0.31}{0.02} & \mstd{0.26}{0.02} \\
    \texttt{puzzle-3x3-noisy}             & \mstd{0.37}{0.06} & \abcell{\mstd{0.95}{0.01}} & \mstd{0.44}{0.15} \\
    \texttt{puzzle-4x4-noisy}             & \mstd{0.03}{0.01} & \abcell{\mstd{0.33}{0.11}} & \mstd{0.24}{0.02} \\
    \bottomrule
  \end{tabular}
  \end{small}
\end{table}

The comparison is reported in \cref{sec:experiments}.

\section{Related Works}
\label{app:related}

\subsection{Goal-conditioned Implicit Q-Learning (GCIQL)}

Implicit Q-Learning (IQL) \cite{kostrikov2021offline} stabilizes offline RL by avoiding queries to
out-of-distribution (OOD) actions through two key components: a state-value function $V_\psi(s)$
and an action-value function $Q_\vartheta(s,a)$. The value functions are trained via:

\begin{equation}
    \mathcal{L}_Q(\vartheta) = \mathbb{E}_{(s,a,s')\sim\mathcal{D}}\left[\left(r(s,a) + \gamma V_\psi(s') - Q_\vartheta(s,a)\right)^2\right],
\end{equation}

\begin{equation}
    \mathcal{L}_V(\psi) = \mathbb{E}_{(s,a)\sim\mathcal{D}}\left[L_2^\tau\left(Q_{\bar{\vartheta}}(s,a) - V_\psi(s)\right)\right],
\end{equation}

where  $L_2^\tau(x) = |\tau - \mathds{1}(x<0)|x^2$  and  $\tau \in [0.5,1)$  controls conservatism
(higher $\tau$ prioritizes optimistic returns), and $\bar{\vartheta}$ are the parameters of the
target $Q$ network. The policy  $\pi_\varphi(a|s)$  is then extracted via advantage-weighted
regression (AWR) \cite{peng2019advantage, wang2020critic}:

\begin{equation}
    J_\pi(\varphi) = \mathbb{E}_{(s,a)\sim\mathcal{D}}\left[\exp\left(\beta \cdot A(s,a)\right)\log\pi_\varphi(a|s)\right],
\end{equation}

with  $A(s,a) = Q_\vartheta(s,a) - V_\psi(s)$, and $\beta$ is the inverse temperature parameter.

For goal-conditioned RL, IQL is extended to learn a goal-conditioned state-value function
$V_\psi(s,g)$, preserving IQL's key advantage of stable value learning without requiring explicit
$Q$-function evaluations on out-of-distribution actions \cite{xu2023implicit}.
GCIVL, as used in OGBench \citep{park2024ogbench}, is the value-only counterpart: it fits
$V_\theta(s,g)$ directly with the expectile loss of \cref{eq:iql_value_loss} and extracts with AWR,
so no action-value function is ever formed.

\subsection{Option-aware Temporally Abstracted Value (OTA)}

HIQL~\citep{park2023hiql} addresses long horizons by introducing a hierarchy over subgoals, but
still relies on flat temporal-difference updates to a high-level value. OTA~\citep{ota2025option}
instead bakes temporal abstraction directly into the Bellman operator by learning
\emph{option-aware} values: for an option $o$ that lasts $k(o)$ steps, the high-level Bellman
target becomes
\begin{equation}
  V(s,g)
  \approx
  \mathbb{E}\bigl[
    r^{(o)}(s,g) + \gamma^{k(o)} V(s',g)
  \bigr],
\end{equation}
where $r^{(o)}$ is the cumulative option reward and $s'$ is the option-termination state. Each
update contracts the effective horizon from $d^\star(s,g)$ to roughly $d^\star(s,g)/k(o)$, so
value differences and the corresponding high-level advantages are computed over multi-step options
rather than single primitive actions. This leads to more stable high-level signals and better
long-horizon stitching on OGBench, at the cost of committing to a particular temporal abstraction
schedule.

\subsection{Quasimetric representations and MQE-style methods}

Recent work views goal-conditioned value learning as estimating an asymmetric ``distance''
$d(s,g)$ between states and goals. Quasimetric approaches~\citep{wang2023optimal,myers2025offline}
fit such distances so that local relations determine longer-horizon reachability.
Multistep quasimetric estimation (MQE)~\citep{zheng2025mqe} regresses them with multistep
returns: instead of bootstrapping only from immediate successors, it regresses
\begin{equation}
  d_\theta(s,g)
  \approx
  \mathbb{E}\Bigl[
    \sum_{t=0}^{K-1} c(s_t,g)
    + d_\theta(s_K,g)
  \,\Bigm|\,
    s_0=s
  \Bigr]
\end{equation}
for random horizons $K$, while encouraging triangle-like inequalities
$d_\theta(s,g)\le d_\theta(s,\tilde g)+d_\theta(\tilde g,g)$ for sampled pivots $\tilde g$. This
multistep quasimetric estimation (MQE) improves horizon generalization - including long-horizon
stitching in visual domains - but typically requires stronger structural assumptions on the value
landscape than local TD methods and can be sensitive to misspecification of the quasimetric prior.

\subsection{Conservative goal-conditioned implicit V-learning}

GCIQL- and GCIVL-style methods can overestimate values for \emph{unconnected} state-goal pairs
produced by cross-trajectory pairing, because expectile regression is optimistic.
CGCIVL~\citep{ke2025conservative} introduces conservative penalties on such pairs together with a
quasimetric formulation. Concretely, for a learned value or distance $v_\theta(s,g)$ and a
connectivity indicator $c(s,g)\in\{0,1\}$ (reachable from $\mathcal{D}$), the loss augments Bellman
terms with
\begin{equation}
  \mathcal{L}_{\mathrm{cons}}(\theta)
  =
  \lambda
  \,\mathbb{E}_{(s,g)\sim p_{\mathrm{pair}}}
  \bigl[
    (1-c(s,g))\,\bigl(\max\{0, v_\theta(s,g)-\delta\}\bigr)^2
  \bigr],
\end{equation}
penalizing large estimates on likely-unreachable pairs. This improves robustness on goal-stitching
tasks in OGBench, but depends on correctly identifying or regularizing unreachable pairs and still
operates on scalar values rather than local action-effect structure.

\subsection{Contrastive representation learning (CRL)}

Contrastive RL methods treat goal-conditioned control as a representation learning problem: they
learn embeddings so that inner products between state(-action) and goal features approximate a
goal-conditioned value or reachability score~\citep{eysenbach2022contrastive}. A typical loss takes
the form
\begin{equation}
  \mathcal{L}_{\mathrm{CRL}}
  =
  -\mathbb{E}\Biggl[
    \log
    \frac{
      \exp\bigl(\phi(s,a)^\top\psi(g^+)/\varsigma\bigr)
    }{
      \sum_{g'\in\mathcal{N}}
      \exp\bigl(\phi(s,a)^\top\psi(g')/\varsigma\bigr)
    }
  \Biggr],
\end{equation}
where $(s,a,g^+)$ is a positive triple and $\mathcal{N}$ is a set of negatives. Policies then act
by choosing actions whose embeddings are closest to the goal embedding. These approaches can learn
powerful, task-agnostic representations from unlabeled trajectories, but the contrastive loss is
global rather than local in the sense of our dual advantage field: it encourages correct ordering
over large batches of positive and negative pairs without explicitly privileging one-step
action-induced displacements in representation space.

\section{Additional Environment and Evaluation Details}
\label{app:experiments-details}

\paragraph{OGBench environments.}
We evaluate on goal-conditioned offline reinforcement learning tasks from OGBench
\citep{park2024ogbench}. OGBench is designed to test several capabilities that are
central to offline GCRL, including long-horizon reasoning, trajectory stitching,
generalization to unseen goals, robustness to suboptimal data, and control under
stochasticity. In our main experiments, we focus on the state-based locomotion and
manipulation tasks used in prior work.

The locomotion tasks include maze-style navigation domains such as \texttt{pointmaze},
\texttt{antmaze}, and \texttt{humanoidmaze}, as well as \texttt{antsoccer}. These tasks
require the agent to reach target goal states from diverse initial configurations using
only offline data. The difficulty varies with maze size, agent morphology, and dataset
coverage. In particular, \texttt{humanoidmaze} requires full-body control and therefore
combines low-level locomotion with long-horizon navigation, while \texttt{antsoccer}
additionally requires controlling a ball while navigating.

The manipulation tasks include \texttt{cube}, \texttt{scene}, and \texttt{puzzle}.
The \texttt{cube} environments test basic object manipulation through pick-and-place,
stacking, swapping, and rearrangement of colored cubes. The \texttt{scene} environment
contains multiple interacting objects, such as a cube, drawer, window, and buttons, and
therefore requires sequencing several atomic behaviors to achieve the desired goal
configuration. The \texttt{puzzle} environments instantiate a robotic version of the
Lights Out puzzle, where pressing one button changes the state of neighboring buttons.
These tasks are particularly challenging because the agent must combine continuous
robotic control with combinatorial generalization over many possible configurations.

\paragraph{Dataset variants.}
For each environment, OGBench provides multiple dataset variants that differ in coverage,
trajectory quality, and the extent to which successful behavior can be recovered directly
from the dataset. In maze-style locomotion tasks, \texttt{navigate} datasets are collected
from noisy expert policies that traverse the environment, while \texttt{stitch} datasets
contain shorter trajectory segments and require the policy to compose partial behaviors
into longer goal-reaching trajectories. Some locomotion domains also provide
\texttt{explore} datasets, which contain highly exploratory and substantially suboptimal
trajectories.

For manipulation tasks, OGBench provides \texttt{play} and \texttt{noisy} datasets.
The \texttt{play} datasets contain natural interaction trajectories generated by scripted
policies with temporally correlated behavior. These datasets often contain useful local
skills but do not necessarily demonstrate each evaluation task end-to-end. The
\texttt{noisy} datasets are collected with larger, less structured exploration noise,
which increases state-action coverage but also makes the data more suboptimal. Together,
these dataset variants test whether an offline GCRL method can learn useful local
behaviors, compose them over long horizons, and remain robust when the data are
imperfect or only partially aligned with the evaluation goals.

\paragraph{Evaluation protocol.}
We follow the standard OGBench protocol and report success-based performance on each
task. For a method $m$, environment $e$, and random seed $r$, let
$s_{m,e,r} \in [0,1]$ denote the resulting success rate, averaged over the evaluation
episodes and goals for that environment. Higher values indicate better goal-reaching
performance. Unless otherwise stated, all methods are trained purely offline on the
provided datasets and are evaluated without additional environment interaction during
training.

\paragraph{Aggregate metrics with RLiable.}
In addition to per-environment results, we report aggregate statistics using the RLiable
evaluation framework \citep{rliable_agarwal2021}. RLiable is useful in the few-seed
regime because it summarizes performance across tasks while also quantifying uncertainty
with stratified-bootstrap confidence intervals. Importantly, RLiable does not discard
``noisy'' runs or remove experiments. Instead, it estimates how sensitive aggregate
conclusions are to the finite set of tasks and random seeds.

Let $S_m = \{s_{m,e,r}\}_{e,r}$ denote the collection of scores for method $m$ across
environments and seeds. We report the following aggregate metrics:
\begin{align}
    \mathrm{Mean}(m)
    &=
    \frac{1}{|\mathcal{E}| |\mathcal{R}|}
    \sum_{e \in \mathcal{E}} \sum_{r \in \mathcal{R}} s_{m,e,r},
    \\
    \mathrm{Median}(m)
    &=
    \operatorname{median}\left(\{s_{m,e,r}\}_{e,r}\right),
    \\
    \mathrm{IQM}(m)
    &=
    \operatorname{mean}\left(
    \{s_{m,e,r}: s_{m,e,r} \text{ lies between the } 25\text{th and }75\text{th percentiles}\}
    \right),
    \\
    \mathrm{OptimalityGap}(m)
    &=
    \frac{1}{|\mathcal{E}| |\mathcal{R}|}
    \sum_{e \in \mathcal{E}} \sum_{r \in \mathcal{R}}
    \max(0, 1 - s_{m,e,r}).
\end{align}
The interquartile mean (IQM) averages the middle $50\%$ of scores and is therefore less
sensitive to extreme outlier tasks than the mean, while being more statistically efficient
than the median. The optimality gap measures the average shortfall from the maximum
normalized score of $1$; thus, lower values are better. Since our scores are success
rates in $[0,1]$, the optimality gap is directly interpretable as the average remaining
failure mass. If scores are reported as percentages, they are first divided by $100$
before computing the RLiable metrics.

For confidence intervals, we use stratified bootstrap resampling over tasks and seeds.
Each bootstrap replicate preserves the task structure: for every environment, we resample
seeds with replacement and then recompute the aggregate metric on the resampled score
matrix. The reported intervals correspond to the empirical percentiles of the bootstrap
distribution. This procedure avoids treating all scores as exchangeable independent
samples and prevents environments with more runs from dominating the uncertainty estimate.

\section{Additional Results}
\label{app:extra}

We include the additional RLiable \citep{rliable_agarwal2021} plots in
\cref{fig:myimage,fig:rliable_prob_of_improvement}.

\begin{figure}[!t]
    \centering
    \includegraphics[width=0.62\linewidth]{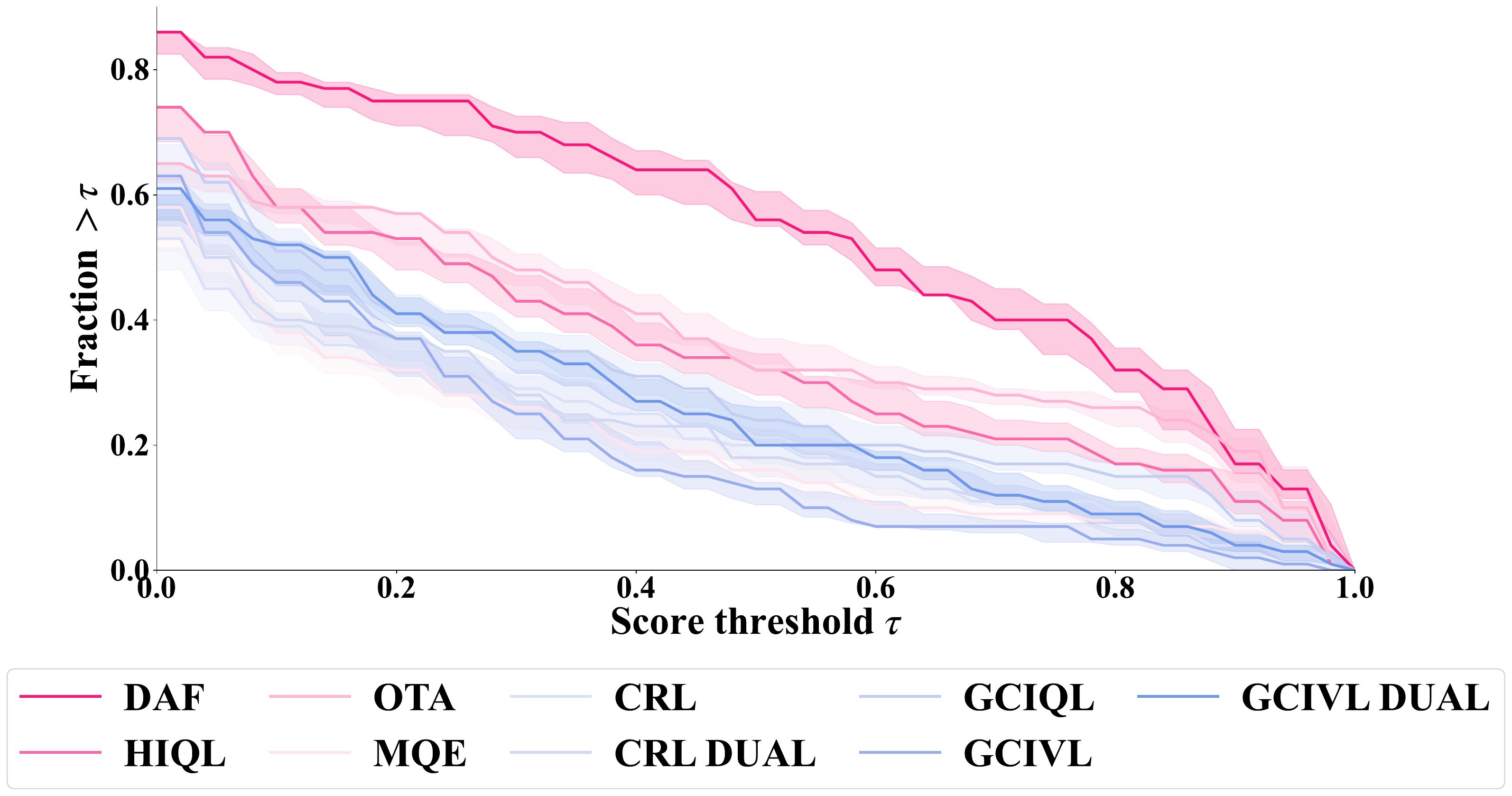}
    \caption{
    \textbf{Performance profile} across all tasks and environments. DAF achieves a better
    distribution of scores than the baselines across the OGBench evaluation suite: the fraction of
    runs exceeding a given normalized score is higher at every threshold.
    }
    \label{fig:myimage}
\end{figure}

\begin{figure}[!t]
    \centering
    \includegraphics[width=0.98\linewidth]{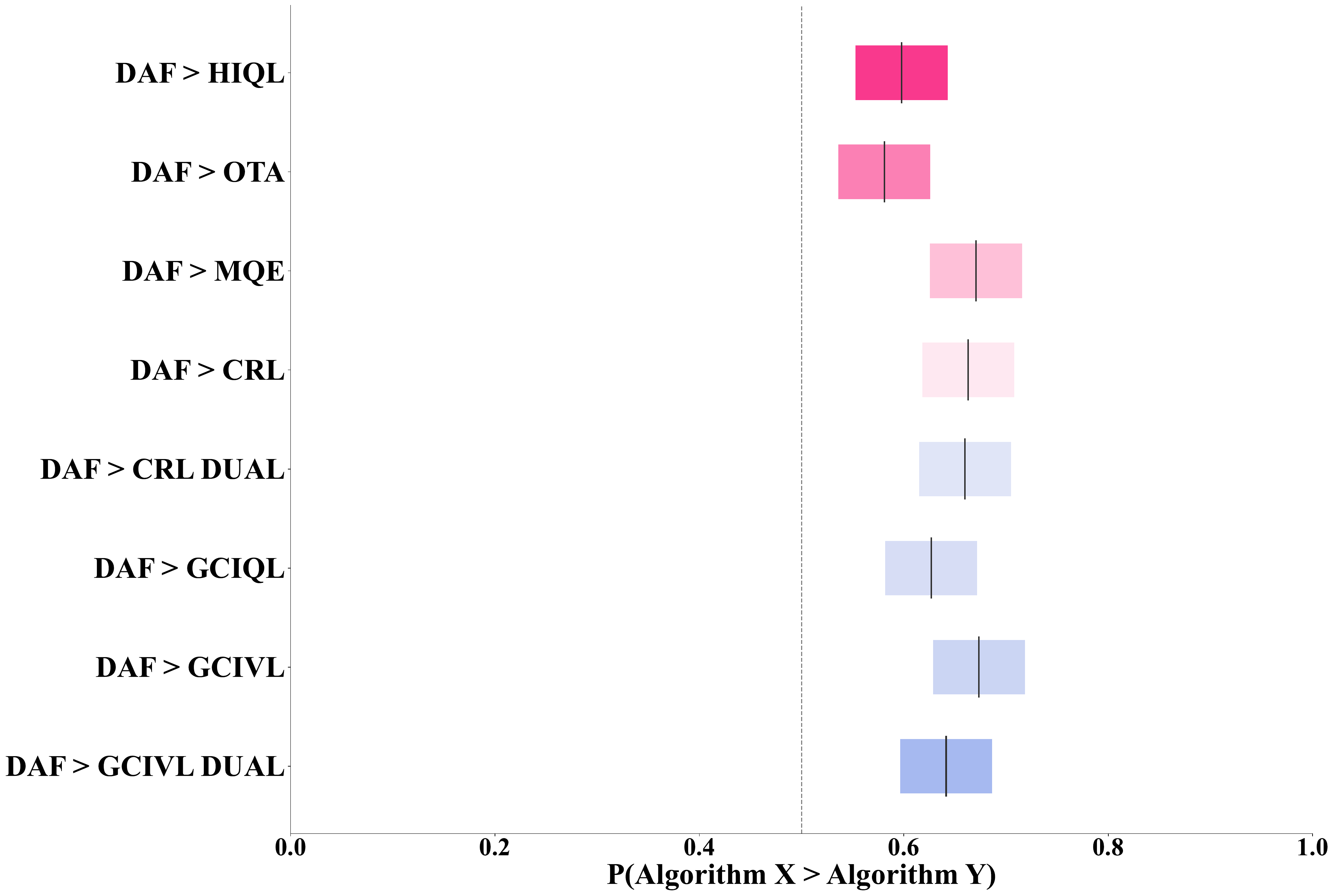}
    \caption{
    \footnotesize
    \textbf{RLiable Probability of Improvement.}
    }
    \label{fig:rliable_prob_of_improvement}
\end{figure}

\subsection{Deterministic-to-stochastic transfer}
\label{app:stochastic}

Because $u_\xi$ estimates a conditional expectation, a multi-modal transition kernel makes it
predict an averaged feature displacement, and the local direction could in principle become
uninformative. \texttt{antmaze-teleport} is constructed to create exactly this situation: certain
state-action pairs enter a teleporter that emits the agent at a randomly chosen destination,
making the kernel genuinely multi-modal \citep{park2024ogbench}.

We isolate the effect by following the OGBench protocol of comparing \texttt{antmaze-large} with
\texttt{antmaze-teleport}, which have the same size and differ only by the stochastic teleporters.
\Cref{tab:stochastic} trains DAF, HIQL and OTA on \texttt{antmaze-large} and compares against the
\texttt{antmaze-teleport} results of \cref{tab:ogbench_main}. On \texttt{stitch}, DAF retains
$89\%$ of its deterministic-maze performance against $58\%$ for HIQL and $44\%$ for OTA; on
\texttt{navigate} all three degrade similarly. In absolute terms DAF is best on
\texttt{teleport-stitch} and fourth of nine on \texttt{teleport-navigate}.

This is consistent with the OGBench observation that optimistically biased methods struggle under
stochasticity: DAF mitigates it through the AFU pessimism on the value side
(\cref{tab:ogbench_ablation} shows the cost of removing it) and through dual embeddings that are
fit to whole reachability relations rather than to single successors.

\begin{table}[H]
  \centering
  \caption{Transfer from deterministic to stochastic dynamics. \texttt{antmaze-large} and
  \texttt{antmaze-teleport} are the same size and differ only by the teleporters, so the retention
  ratio isolates the cost of stochasticity. Retention is the \texttt{teleport} score divided by
  the \texttt{large} score.}
  \label{tab:stochastic}
  \begin{small}
  \begin{tabular}{@{}llccc@{}}
    \toprule
    \textbf{Dataset} & \textbf{Method} & \texttt{antmaze-large} & \texttt{antmaze-teleport} & \textbf{Retention} \\
    \midrule
    \multirow{3}{*}{navigate}
      & DAF  & \mstd{0.93}{0.02} & \mstd{0.51}{0.08} & $55\%$ \\
      & HIQL & \mstd{0.95}{0.01} & \mstd{0.46}{0.03} & $48\%$ \\
      & OTA  & \mstd{0.94}{0.01} & \mstd{0.53}{0.03} & $56\%$ \\
    \midrule
    \multirow{3}{*}{stitch}
      & DAF  & \mstd{0.56}{0.04} & \mstd{0.50}{0.05} & $\mathbf{89\%}$ \\
      & HIQL & \mstd{0.65}{0.04} & \mstd{0.38}{0.02} & $58\%$ \\
      & OTA  & \mstd{0.87}{0.03} & \mstd{0.38}{0.01} & $44\%$ \\
    \bottomrule
  \end{tabular}
  \end{small}
\end{table}

\section{AFU-style coupling of the bilinear value and dual score}
\label{app:afu-details}

This section provides the actor-free coupling we use between the bilinear value
$V_\theta(s,g)=\psi_\theta(s)^\top\phi_\theta(g)$ and the dual score, following the
separation of roles emphasized by \citet{perrin2024afu}. The policy parameters do
not receive gradients through this objective; policy learning uses only the
weighted regression step of \cref{sec:method_awr}.

\paragraph{Surrogate dual score for the coupling.}
The main text defines the raw dual score $z_\theta$ in~\eqref{eq:raw-dual}. In the
AFU objective below it is convenient to use a non-positive surrogate
\begin{equation}
  \label{eq:atilde-appendix}
  \widetilde A_\theta(s,a,g)
  \;:=\;
  h\!\left(z_\theta(s,a,g)\right),\qquad h:\mathbb{R}\to(-\infty,0],
\end{equation}
where $h$ is any monotone transformation used in implementation to keep the coupling
term bounded on the optimistic side while preserving action ordering. In our experiments we use
$h(x)=-\operatorname{softplus}(-x)$, which is nearly the identity for strongly negative $x$ and
saturates for positive $x$. The same
$z_\theta$ can still be used directly in advantage-weighted regression, as in the
main text; \eqref{eq:atilde-appendix} is only required for the piecewise coupling
with~$V_\theta$.

\paragraph{Scalar Bellman target.}
Let
\begin{equation}
  T(s,a,g) \;:=\; r(s,g) + \gamma\,V^{\mathrm{tgt}}_\theta(s',g)\,,
\end{equation}
with $V^{\mathrm{tgt}}_\theta$ a slowly updated target network for the bilinear head.

\paragraph{Conditional scaling of $V_\theta$.}
Let $U=\mathbbm{1}[V_\theta+\widetilde A_\theta < T]$ and $\rho\in(0,1)$. Define
\begin{equation}
  \widetilde V \;:=\; (1-\rho U)\,V_\theta + \rho U\,\mathrm{stopgrad}(V_\theta)\,.
\end{equation}
When the optimistic sum $V_\theta+\widetilde A_\theta$ falls short of the
Bellman target $T$, the mask down-weights direct updates to $V_\theta$ so that
$\widetilde A_\theta$ can absorb slack in the near-optimistic regime.

\paragraph{Piecewise coupling loss.}
With $x=\widetilde V - T$ and $y=\widetilde A_\theta$, set
\begin{equation}
  Z(x,y) \;=\;
  \begin{cases}
    (x+y)^2\,, & x\ge 0\,,\\
    x^2 + y^2\,, & x<0\,,
  \end{cases}
  \qquad
  \mathcal{L}_{\mathrm{afu}} \;=\; \mathbb{E}_{\mathcal{D}}\bigl[Z(x,y)\bigr].
  \label{eq:afu-loss}
\end{equation}
Training minimizes $\mathcal{L}_{\mathrm{afu}}$ jointly over the parameters
of $V_\theta$ (equivalently $\psi_\theta$ and, where tied, $\phi_\theta$) and of
the heads that define $\widetilde A_\theta$ (including $u_\xi$ and $\phi_\theta$
as used in $z_\theta$). The asymmetric split between $x\ge 0$ and $x<0$ mirrors
the AFU construction: pessimistic errors on $V$ and the dual score are not
forced to cancel spuriously when the backup is optimistic.

\paragraph{Feature dynamics auxiliary loss.}
The loss $\mathcal{L}_{\mathrm{ae}}$ in~\eqref{eq:ae-mse} complements the coupling
above: the AFU-style term enforces Bellman consistency between $V_\theta$ and
$\widetilde A_\theta$, while $\mathcal{L}_{\mathrm{ae}}$ grounds $u_\xi$ in
explicit one-step feature dynamics on the offline dataset. Removing either one degrades
performance (\cref{tab:ogbench_ablation}).

\clearpage
\section{Theoretical Analysis}
\label{app:theory-daf}

This section establishes three properties of Dual Advantage Fields.
\Cref{subsec:exact_adv} shows that under exact representability the DAF score \emph{is} the
Bellman advantage, so DAF is a valid policy-improvement operator.
\Cref{subsec:amortization} quantifies why the action effect is learned rather than read off a
single transition. \Cref{subsec:didactic} analyses a didactic one-dimensional example, comparing
DAF against flat and hierarchical value-difference extraction under a learned, noisy goal
embedding.

\subsection{DAF as an exact policy-improvement signal}
\label{subsec:exact_adv}

Fix a goal $g$ and consider the goal-conditioned MDP with reward $r_g(s,a):=r(s,a,g)$. For a
policy $\pi$, recall the Bellman advantage \eqref{eq:adv-bellman}. Only the ordering of actions at
a given state matters for improvement, not the absolute level of $V^\pi$.

\begin{calloutimportant}
\begin{assumption}[Exact dual realizability]
\label{asm:realizable}
For the policy $\pi$ under consideration,
\begin{enumerate}[label=(\roman*),leftmargin=2em,itemsep=0pt,topsep=1pt]
  \item \textbf{Value realizability:} there exist $\psi:\mathcal{S}\to\mathbb{R}^d$ and
  $\phi:\mathcal{G}\to\mathbb{R}^d$ with $V^\pi(s,g)=\psi(s)^\top\phi(g)$ for all
  $s\in\mathcal{S}$, $g\in\mathcal{G}$;
  \item \textbf{Exact action effect:} the action-effect model satisfies
  $u(s,a)=\mathbb{E}_{s'\sim P(\cdot\mid s,a)}\bigl[\gamma\psi(s')-\psi(s)\bigr]$ for all
  $(s,a)$.
\end{enumerate}
\end{assumption}
\end{calloutimportant}

Define the population DAF score $D^\pi(s,a,g) := r(s,a,g)+u(s,a)^\top\phi(g)$.

\begin{lemma}[The score is the advantage]
\label{lem:score-is-advantage}
Under \cref{asm:realizable}, $D^\pi(s,a,g)=A^\pi(s,a,g)$ for all $(s,a,g)$.
\end{lemma}

\begin{proof}
By \cref{asm:realizable}(i), $\gamma V^\pi(s',g)-V^\pi(s,g)=\phi(g)^\top(\gamma\psi(s')-\psi(s))$
for every $s'$. Taking the expectation over $s'\sim P(\cdot\mid s,a)$ and using linearity together
with \cref{asm:realizable}(ii) gives
$\mathbb{E}_{s'}[\gamma V^\pi(s',g)-V^\pi(s,g)]=u(s,a)^\top\phi(g)$. Adding $r(s,a,g)$ and
comparing with \eqref{eq:adv-bellman} yields the claim.
\end{proof}

\begin{proposition}[DAF local policy improvement]
\label{prop:daf_improvement}
Let $\pi^+$ be any goal-conditioned policy satisfying
$\mathbb{E}_{a\sim\pi^+(\cdot\mid s,g)}[D^\pi(s,a,g)]\ge 0$ for all $s,g$. Then, under
\cref{asm:realizable}, $V^{\pi^+}(s,g)\ge V^\pi(s,g)$ for all $s,g$.
\end{proposition}

\begin{proof}
By \cref{lem:score-is-advantage} the hypothesis reads
$\mathbb{E}_{a\sim\pi^+(\cdot\mid s,g)}[A^\pi(s,a,g)]\ge0$, which is exactly
\begin{equation*}
  (T_{\pi^+}V^\pi)(s,g)-V^\pi(s,g)
  = \mathbb{E}_{a\sim\pi^+(\cdot\mid s,g)}[A^\pi(s,a,g)] \;\ge\; 0,
\end{equation*}
where $T_{\pi^+}$ is the Bellman operator of $\pi^+$. Hence $T_{\pi^+}V^\pi\ge V^\pi$ pointwise,
and by monotonicity of $T_{\pi^+}$, $T_{\pi^+}^{k}V^\pi\ge V^\pi$ for every $k\ge1$. Letting
$k\to\infty$ and using that $T_{\pi^+}$ is a $\gamma$-contraction with fixed point $V^{\pi^+}$
gives $V^{\pi^+}\ge V^\pi$.
\end{proof}

\paragraph{AWR satisfies the hypothesis.}
The exponential reweighting of \eqref{eq:awr-weights} is one such $\pi^+$ in the exact on-policy
case. Let
$\pi_\alpha^+(a\mid s,g)\propto\pi(a\mid s,g)\exp(\alpha D^\pi(s,a,g))$ with $\alpha\ge0$ and
$Z(\alpha)=\mathbb{E}_{a\sim\pi}[\exp(\alpha D^\pi(s,a,g))]$. Then
\begin{equation*}
  \frac{d}{d\alpha}\log Z(\alpha)=\mathbb{E}_{a\sim\pi_\alpha^+}[D^\pi(s,a,g)],
  \qquad
  \frac{d^2}{d\alpha^2}\log Z(\alpha)=\operatorname{Var}_{a\sim\pi_\alpha^+}[D^\pi(s,a,g)]\ge0 .
\end{equation*}
Since $\log Z$ is convex and its derivative at $\alpha=0$ equals
$\mathbb{E}_{a\sim\pi}[D^\pi(s,a,g)]=\mathbb{E}_{a\sim\pi}[A^\pi(s,a,g)]=0$, the derivative is
nonnegative for all $\alpha\ge0$. Hence
$\mathbb{E}_{a\sim\pi_\alpha^+}[D^\pi(s,a,g)]\ge0$ and \cref{prop:daf_improvement} applies.

\begin{corollary}[Exact DAF policy iteration]
\label{cor:daf_policy_iteration}
In a finite discounted goal-conditioned MDP, suppose \cref{asm:realizable} holds at every
iteration $k$ for $V^{\pi_k}$ and define
$\pi_{k+1}(\cdot\mid s,g)\in\argmax_{\pi'}\mathbb{E}_{a\sim\pi'(\cdot\mid s,g)}[D^{\pi_k}(s,a,g)]$.
Then $\pi_{k+1}$ is the standard greedy policy-improvement step with respect to $Q^{\pi_k}$, and
repeated exact DAF improvement is policy iteration and converges to an optimal goal-conditioned
policy.
\end{corollary}

\begin{proof}
By \cref{lem:score-is-advantage}, $D^{\pi_k}=A^{\pi_k}=Q^{\pi_k}-V^{\pi_k}$, and $V^{\pi_k}(s,g)$
does not depend on $a$, so maximizing $D^{\pi_k}$ over actions is equivalent to maximizing
$Q^{\pi_k}$. The result follows from classical convergence of policy iteration in finite
discounted MDPs, applied separately for each $g$.
\end{proof}

\paragraph{Relation to hierarchical policy classes.}
Let $\Pi$ be the class of all stationary goal-conditioned primitive-action policies and
$\Pi_{\mathrm{hier}}\subseteq\Pi$ any hierarchically constrained subclass induced by a fixed
subgoal or option parameterization. Since
$\pi^\star\in\argmax_{\pi\in\Pi}V^\pi$ satisfies
$V^{\pi^\star}(s,g)\ge\sup_{\pi\in\Pi_{\mathrm{hier}}}V^\pi(s,g)$ for all $s,g$,
\cref{cor:daf_policy_iteration} implies that in the exact realizable limit DAF policy iteration
reaches a policy at least as good as the best policy in any fixed hierarchical class. This is a
\emph{representational} statement and not a finite-sample one: hierarchy can accelerate learning
by reducing the effective horizon, which is why we use it, but a fixed hierarchy can also exclude
the optimal primitive-action policy.

\subsection{Why learn the action-effect module?}
\label{subsec:amortization}

\Cref{asm:realizable}(ii) asks for a conditional expectation. A single dataset transition provides
one draw from $P(\cdot\mid s,a)$, so the obvious alternative - scoring with
$\widehat z(s,a,s',g)=(\gamma\psi(s')-\psi(s))^\top\phi(g)$, which is \eqref{eq:adv-dual} up to
the reward term - is an unbiased but noisy estimate of the same quantity. The following makes the
cost precise.

\begin{proposition}[One-sample scoring can invert the action ranking]
\label{prop:amortization}
Fix $s$ and $g$, let $a,b\in\mathcal{A}$ and write
$\sigma_a^2:=\operatorname{Var}_{s'\sim P(\cdot\mid s,a)}\bigl(\phi(g)^\top\psi(s')\bigr)$,
similarly $\sigma_b^2$. Let $m:=D^\pi(s,a,g)-D^\pi(s,b,g)>0$, so that $a$ is genuinely the better
action, and let $\widehat z$ be evaluated at independent draws $s'_a\sim P(\cdot\mid s,a)$ and
$s'_b\sim P(\cdot\mid s,b)$ with the reward term included. Then
\begin{equation}
  \mathbb{E}\bigl[\widehat z(s,a,s'_a,g)\bigr] = D^\pi(s,a,g),
  \qquad
  \operatorname{Var}\bigl(\widehat z(s,a,s'_a,g)\bigr) = \gamma^2\sigma_a^2 ,
  \label{eq:plugin-moments}
\end{equation}
and the probability that the one-sample rule prefers $b$ to $a$ satisfies
\begin{equation}
  \mathbb{P}\bigl[\widehat z(s,b,s'_b,g)\ge\widehat z(s,a,s'_a,g)\bigr]
  \;\le\;
  \frac{\gamma^2(\sigma_a^2+\sigma_b^2)}{m^2},
  \label{eq:plugin-chebyshev}
\end{equation}
with the bound replaced by the exact value
$\Phi\bigl(-m/(\gamma\sqrt{\sigma_a^2+\sigma_b^2})\bigr)$ when
$\phi(g)^\top\psi(s')$ is Gaussian under both kernels, where $\Phi$ is the standard normal
c.d.f. By contrast, the population minimizer of $\mathcal{L}_{\mathrm{ae}}$ \eqref{eq:ae-mse}
satisfies $u_\xi(s,a)=\mathbb{E}_{s'}[\gamma\psi(s')-\psi(s)]$ exactly, so the amortized score
$z_\theta$ of \eqref{eq:raw-dual} prefers $a$ to $b$ with probability one.
\end{proposition}

\begin{proof}
\Cref{eq:plugin-moments} is immediate: $\psi(s)$ and $\phi(g)$ are deterministic given $(s,g)$,
so $\widehat z = r(s,a,g)-\psi(s)^\top\phi(g)+\gamma\,\phi(g)^\top\psi(s')$ is an affine function
of the scalar random variable $\phi(g)^\top\psi(s')$ with slope $\gamma$. For
\cref{eq:plugin-chebyshev}, write $\widehat z(s,a,s'_a,g)=D^\pi(s,a,g)+\xi_a$ and likewise for
$b$, where $\xi_a,\xi_b$ are independent, zero-mean, with variances $\gamma^2\sigma_a^2$ and
$\gamma^2\sigma_b^2$. The event in question is $\xi_b-\xi_a\ge m$, and
$\operatorname{Var}(\xi_b-\xi_a)=\gamma^2(\sigma_a^2+\sigma_b^2)$, so Chebyshev's inequality gives
the bound; under joint Gaussianity $\xi_b-\xi_a$ is normal and the probability is exactly
$\Phi(-m/(\gamma\sqrt{\sigma_a^2+\sigma_b^2}))$. The final claim is the standard fact that the
minimizer of a squared-error regression objective is the conditional mean of its target, here
$\mathbb{E}[\gamma\psi(s')-\psi(s)\mid s,a]$; substituting into \eqref{eq:raw-dual} recovers
$D^\pi$ up to the action-independent reward term, whose ordering is unaffected.
\end{proof}

\subsection{Robustness to learned embedding noise: a didactic example}
\label{subsec:didactic}

We now move to the regime where the representation is learned from finite data and therefore
contains noise, and compare three extraction rules built on the \emph{same} learned value and the
same action-effect model.

\begin{callout}
\begin{assumption}[Line world with nuisance coordinates]
\label{asm:lineworld}
\begin{enumerate}[label=(\roman*),leftmargin=2em,itemsep=0pt,topsep=1pt]
  \item \textbf{Dynamics.} States $s\in\{0,\dots,T\}$ with fixed goal $g=T>0$; actions
  $a\in\{-1,+1\}$ move deterministically to $s+a$; the episode terminates at $g$; the reward is
  $0$ at the goal and $-1$ elsewhere. Hence the optimal action is $+1$ for all $s<T$ and
  $V^\star(s,g)=s-T$.
  \item \textbf{Embedding.} $\psi(s)=[\,s,\;1,\;f_1(s),\dots,f_m(s)\,]^\top\in\mathbb{R}^{m+2}$,
  where the $f_i$ are bounded with well-defined first and second differences. The first two
  coordinates are \emph{essential}; the remaining $m$ are \emph{nuisance} coordinates, irrelevant
  to control (\eg\ visual textures or lighting gradients).
  \item \textbf{True goal embedding.} $\phi^\star(g)=[\,1,\,-T,\,0,\dots,0\,]^\top$, so that
  $\psi(s)^\top\phi^\star(g)=s-T=V^\star(s,g)$.
  \item \textbf{Learned goal embedding.} $\phi(g)=\phi^\star(g)+\boldsymbol\varepsilon$ with
  $\boldsymbol\varepsilon=[\,0,0,\eta_1,\dots,\eta_m]^\top$ and
  $\eta_i\sim\mathcal{N}(0,\sigma_i^2)$ independent: the temporal-difference loss constrains the
  nuisance coefficients only weakly, so the noise is confined to them. The essential coordinates
  are assumed exact; allowing noise there does not change the conclusions but obscures the
  algebra. Consequently
  $\widehat V(s,g)=\psi(s)^\top\phi(g)=s-T+\sum_{i=1}^m\eta_i f_i(s)$.
  \item \textbf{Subgoal embedding.} Subgoals are embedded by the same map with independent noise,
  $\phi(s_{\mathrm{sub}})=\phi^\star(s_{\mathrm{sub}})+\tilde{\boldsymbol\varepsilon}$,
  $\tilde\eta_i\sim\mathcal{N}(0,\sigma_i^2)$ independent of $\eta_i$.
  \item \textbf{Exact action effect.} $u(s,a)=\gamma\psi(s+a)-\psi(s)$, which is attainable here
  because the dynamics are deterministic and transitions are abundant.
\end{enumerate}
\end{assumption}
\end{callout}

The three rules under \cref{asm:lineworld} are: \emph{flat value-difference}
$a_{\mathrm{V}}(s)=\argmax_{a}\widehat V(s+a,g)$, which is the extraction implicit in HIQL's flat
baseline; \emph{DAF} $a_{\mathrm{DAF}}(s)=\argmax_{a}u(s,a)^\top\phi(g)$; and
\emph{hierarchical} extraction, which first picks
$s_{\mathrm{sub}}=\argmax_{x\in\{s+k,s-k\}}[\widehat V(x,g)-\widehat V(s,g)]$ for a stride
$k\ge2$, then applies the flat rule toward the chosen subgoal,
$a_{\ell}(s)=\argmax_{a}\widehat V(s+a,s_{\mathrm{sub}})$. Write
$\Delta f(s):=f(s+1)-f(s-1)$ for the first difference.

The first observation is a negative one, and it is worth being explicit about because it bounds
what the didactic model can support.

\begin{lemma}[In this model, DAF and flat extraction make the same decision]
\label{lem:daf-equals-flat}
Under \cref{asm:lineworld}, the DAF and flat decision margins satisfy
\begin{equation}
  \Delta_{\mathrm{DAF}}(s)
  := \bigl(u(s,+1)-u(s,-1)\bigr)^\top\phi(g)
  = \gamma\,\bigl(\widehat V(s+1,g)-\widehat V(s-1,g)\bigr)
  = \gamma\,\Delta_{\mathrm{V}}(s),
\end{equation}
so the two rules select the same action at every state and have identical error probabilities.
\end{lemma}

\begin{proof}
$u(s,+1)-u(s,-1)=(\gamma\psi(s+1)-\psi(s))-(\gamma\psi(s-1)-\psi(s))=\gamma(\psi(s+1)-\psi(s-1))$.
Taking the inner product with $\phi(g)$ and using
$\widehat V(x,g)=\psi(x)^\top\phi(g)$ gives the identity. Since $\gamma>0$, the two margins have
the same sign.
\end{proof}

\begin{remark}
\label{rem:no-second-difference}
\Cref{lem:daf-equals-flat} rules out the tempting claim that DAF filters representation noise
through \emph{second} differences $f_i(s+1)+f_i(s-1)-2f_i(s)$: the action-effect vectors differ by
a first difference, not a second one. In this idealized model, therefore, the advantage of DAF
over flat extraction is \emph{not} a noise-filtering effect. The two places where DAF provably
differs are \cref{prop:amortization}, which concerns stochastic kernels and is invisible here by
\cref{asm:lineworld}(i), and the comparison with hierarchy below, which concerns the number of
noisy decisions that have to be made.
\end{remark}

\begin{proposition}[Error probabilities of the three rules]
\label{prop:errors}
Under \cref{asm:lineworld}, for any $s\in\{1,\dots,T-1\}$ and stride $k$, with $\Phi$ the standard
normal c.d.f.,
\begin{equation}
\begin{aligned}
  \varepsilon_{\mathrm{DAF}}(s) = \varepsilon_{\mathrm{flat}}(s)
  = \varepsilon_{\mathrm{low}}(s)
  &= \Phi\Bigl(-2\bigl(\textstyle\sum_i \sigma_i^2(\Delta f_i(s))^2\bigr)^{-1/2}\Bigr),
  \\
  \varepsilon_{\mathrm{high}}(s)
  &= \Phi\Bigl(-2k\bigl(\textstyle\sum_i \sigma_i^2(f_i(s+k)-f_i(s-k))^2\bigr)^{-1/2}\Bigr).
\end{aligned}
\end{equation}
\end{proposition}

\begin{proof}
Each decision margin is an affine function of the Gaussian vector $\boldsymbol\eta$ (or
$\tilde{\boldsymbol\eta}$ for the low level) and is therefore Gaussian. Direct computation gives
$\Delta_{\mathrm{V}}(s)=2+\sum_i\eta_i\Delta f_i(s)$,
$\Delta_{\mathrm{high}}(s)=2k+\sum_i\eta_i(f_i(s+k)-f_i(s-k))$ and
$\Delta_{\mathrm{low}}(s)=2+\sum_i\tilde\eta_i\Delta f_i(s)$ conditioned on the subgoal $s+k$
being chosen. Each rule errs exactly when its margin is negative, and
$\mathbb{P}[\mathcal{N}(\mu,\varsigma^2)<0]=\Phi(-\mu/\varsigma)$. The first chain of equalities
uses \cref{lem:daf-equals-flat} and the fact that $\eta$ and $\tilde\eta$ have the same variances.
\end{proof}

\begin{callout}
\begin{assumption}[Hierarchical failure model]
\label{asm:hier-indep}
The hierarchical decision at $s$ is counted as correct only if \emph{both} stages are correct:
(i) the two stages' noise sources are independent, which holds under
\cref{asm:lineworld}(v); and (ii) when the high level selects the wrong subgoal, the resulting
primitive action is wrong - a low-level error is not credited with accidentally recovering the
correct action.
\end{assumption}
\end{callout}

\Cref{asm:hier-indep}(ii) is the conservative bookkeeping convention used in the hierarchical
analysis of \citet[Proposition~4.1]{park2023hiql}; without it, the two errors could cancel and the
hierarchical error would be the symmetric difference of the two events rather than their union.

\begin{proposition}[Exact hierarchical error and two-sided bounds]
\label{prop:hier-exact}
Under \cref{asm:lineworld,asm:hier-indep}, the hierarchical error probability is
\begin{equation}
  \varepsilon_{\mathrm{hier}}(s)
  = \varepsilon_{\mathrm{high}}(s) + \bigl(1-\varepsilon_{\mathrm{high}}(s)\bigr)\,
    \varepsilon_{\mathrm{low}}(s),
  \label{eq:hier-exact}
\end{equation}
and consequently
\begin{equation}
  \max\{\varepsilon_{\mathrm{high}}(s),\varepsilon_{\mathrm{low}}(s)\}
  \;\le\;
  \varepsilon_{\mathrm{hier}}(s)
  \;\le\;
  \varepsilon_{\mathrm{high}}(s)+\varepsilon_{\mathrm{low}}(s).
  \label{eq:hier-twosided}
\end{equation}
In particular $\varepsilon_{\mathrm{DAF}}(s)=\varepsilon_{\mathrm{low}}(s)\le
\varepsilon_{\mathrm{hier}}(s)$, with strict inequality whenever
$\varepsilon_{\mathrm{high}}(s)>0$ and $\varepsilon_{\mathrm{low}}(s)<1$.
\end{proposition}

\begin{proof}
Let $E_{\mathrm{high}}$ and $E_{\mathrm{low}}$ be the two stage-error events. By
\cref{asm:hier-indep} the hierarchical rule errs on $E_{\mathrm{high}}\cup E_{\mathrm{low}}$, and
by independence
$\mathbb{P}[E_{\mathrm{high}}\cup E_{\mathrm{low}}]
=\varepsilon_{\mathrm{high}}+\varepsilon_{\mathrm{low}}-\varepsilon_{\mathrm{high}}\varepsilon_{\mathrm{low}}$,
which is \cref{eq:hier-exact}. Writing \cref{eq:hier-exact} in the two equivalent forms
$\varepsilon_{\mathrm{high}}+(1-\varepsilon_{\mathrm{high}})\varepsilon_{\mathrm{low}}$ and
$\varepsilon_{\mathrm{low}}+(1-\varepsilon_{\mathrm{low}})\varepsilon_{\mathrm{high}}$ and using
$\varepsilon\in[0,1]$ gives
$\varepsilon_{\mathrm{hier}}\ge\varepsilon_{\mathrm{high}}$ and
$\varepsilon_{\mathrm{hier}}\ge\varepsilon_{\mathrm{low}}$, hence the lower bound in
\cref{eq:hier-twosided}; the upper bound follows by dropping the nonnegative product term. The
final statement combines this with
$\varepsilon_{\mathrm{DAF}}=\varepsilon_{\mathrm{low}}$ from \cref{prop:errors}, and the
inequality is strict because the discarded term
$\varepsilon_{\mathrm{high}}(1-\varepsilon_{\mathrm{low}})$ is then strictly positive.
\end{proof}

\Cref{prop:hier-exact} replaces a union bound with an identity, which matters for the direction
of the comparison: an upper bound on $\varepsilon_{\mathrm{hier}}$ cannot show that hierarchy is
worse than flat extraction, whereas the lower bound
$\varepsilon_{\mathrm{hier}}\ge\varepsilon_{\mathrm{low}}=\varepsilon_{\mathrm{DAF}}$ can.

\begin{remark}[Scope of the didactic model]
\label{rem:didactic-scope}
\Cref{asm:lineworld} grants every rule an accurate value field, corrupted only in
control-irrelevant directions. That is exactly the regime in which hierarchy has nothing to
contribute, since its purpose is to improve the \emph{propagation} of value over long horizons
rather than the reliability of a single decision. The model therefore says something
narrow but useful: whatever benefit hierarchy provides in practice must come from value quality,
not from decision reliability, because in the decision-reliability accounting it strictly loses.
This is consistent with \cref{tab:ogbench_ablation}, where removing the hierarchy helps on the two
tasks whose value fields are apparently already adequate
(\texttt{humanoidmaze-large-stitch}, \texttt{puzzle-4x4-play}) and collapses on
\texttt{antmaze-stitch}, where long-horizon propagation is the bottleneck.
\end{remark}

\section{Compute Resources}
\label{app:compute}

All experiments were performed on servers with a single H100 GPU with 80\,GB of GPU memory,
12 CPU cores, and 244\,GB of RAM. All metrics for the experiments were logged using the
Weights \& Biases platform. Overall, the Weights \& Biases project of the paper had $17{,}359$
tracked experiments at the time of submission and used an estimated ${\sim}407$ days of GPU
compute in total.

\end{document}